\documentclass[final]{siamltexmm}
\usepackage{amsmath, amssymb, amsfonts, url, bm, mathrsfs}
\title{Efficient Point-to-Subspace Query in $\ell^1$ with Application to Robust Object Instance Recognition} 

\author{Ju Sun\thanks{Department of Electrical Engineering, Columbia University, New York, USA (\email{jusun@ee.columbia.edu}). JS gratefully acknowledges support from the Wei Family Private Foundation. YZ and JW were partially supported by ONR N00014-13-1-0492 and Columbia University startup funding.} \and Yuqian Zhang\thanks{Department of Electrical Engineering, Columbia University, New York, USA (\email{yz2409@columbia.edu}).} \and John Wright\thanks{Department of Electrical Engineering, Columbia University, New York, USA (\email{johnwright@ee.columbia.edu}).}}

\begin{document}
\maketitle
\newtheorem{conjecture}[theorem]{Conjecture}
\newtheorem{problem}[theorem]{Problem}
\newtheorem{claim}[theorem]{Claim}
\newtheorem{remark}[subsection]{Remark}
\newtheorem{example}[subsection]{Example}

\newcommand{\eps}{\varepsilon}
\newcommand{\R}{\mathbb{R}}
\newcommand{\Z}{\mathbb{Z}}
\newcommand{\N}{\mathbb{N}}

\renewcommand{\Re}{\R}
\newcommand{\event}{\mc E}

\newcommand{\mb}{\mathbf}
\newcommand{\mc}{\mathcal}
\newcommand{\mf}{\mathfrak}
\newcommand{\mbb}{\mathbb}
\newcommand{\msc}{\mathscr}

\newcommand{\vtrz}{\mathrm{vec}}

\newcommand{\norm}[2]{\left\| #1 \right\|_{#2}}
\newcommand{\innerprod}[2]{\left\langle #1,  #2 \right\rangle}
\newcommand{\prob}[1]{\mbb P\left[ #1 \right]}
\newcommand{\expect}[1]{\mbb E\left[ #1 \right]}
\newcommand{\function}[2]{#1 \left(#2\right)}
\newcommand{\integral}[4]{\int_{#1}^{#2}\; #3\; #4}
\newcommand{\js}[1]{{\color{magenta}{\bf Note: #1}}}
\newcommand{\jw}[1]{{\color{blue}{\bf John: #1}}}

\newcommand{\red}{\color{red}}
\newcommand{\slugmaster}{%
\slugger{siims}{xxxx}{xx}{x}{x--x}}

\begin{abstract}
Motivated by vision tasks such as robust face and object recognition, we consider the following general problem: given a collection of low-dimensional linear subspaces in a high-dimensional ambient (image) space, and a query point (image), efficiently determine the nearest subspace to the query in $\ell^1$ distance. In contrast to the naive exhaustive search which entails large-scale linear programs, we show that the computational burden can be cut down significantly by a simple two-stage algorithm: (1) projecting the query and data-base subspaces into lower-dimensional space by random Cauchy matrix, and solving small-scale distance evaluations (linear programs) in the projection space to locate candidate nearest; (2) with few candidates upon independent repetition of (1), getting back to the high-dimensional space and performing exhaustive search. To preserve the identity of the nearest subspace with nontrivial probability, the projection dimension typically is low-order polynomial of the subspace dimension multiplied by logarithm of number of the subspaces (Theorem~\ref{thm:main}). The reduced dimensionality and hence complexity renders the proposed algorithm particularly relevant to vision application such as robust face and object instance recognition that we investigate empirically. 
\end{abstract}

\begin{keywords}$\ell^1$ point-to-subspace distance, nearest subspace search, Cauchy projection, face recognition, subspace modeling\end{keywords}

\begin{AMS}68U10, 68T45, 68W20, 68T10, 15B52\end{AMS}

\pagestyle{myheadings}
\thispagestyle{plain}
\markboth{Efficient Point-to-Subspace Query in $\ell^1$}{Sun, Zhang, and Wright}

\section{Introduction}
Although visual data reside in very high-dimensional spaces, they often exhibit much lower-dimensional intrinsic structure. Modeling and exploiting this low-dimensional structure is a central goal in computer vision, with impact on applications from low-level tasks such as signal acquistion and denoising to higher-level tasks such as object detection and recognition. 

In face and object recognition alone, many popular, effective techniques can be viewed as searching for the low-dimensional model which best matches the query (test) image (e.g., ~\cite{ho2003clustering, basri2007approximate}). To each object $\mc O$ of interest, we may associate a low-dimensional subset $\mc M \subset \Re^D$, which approximates the set of images of $\mc O$ that can be generated under different physical conditions -- say, varying pose or illumination. Given $n$ objects $\mc O_i$ and their corresponding approximation subsets $\mc M_i$, the recognition problem becomes one of finding the nearest low-dimensional structure. To put it formal, 
\begin{align*}
\mathop{\arg\min}_i d(\mb q,\mc M_i),
\end{align*} 
where $\mb q \in \Re^D$ is the test image, and $d(\cdot,\cdot)$ is some prescribed point-to-set distance function. 

This paradigm is broad enough to encompass very classical work in face recognition \cite{Turk1991-CVPR} and object instance recognition \cite{Murase1995-IJCV}, as well as more recent developments \cite{Cootes2001-PAMI,Blanz2003-PAMI,Wright2009-PAMI}. In situations when sufficient training data are available to accurately fit the $\mc M_i$, it can achieve high recognition rates \cite{Wagner2012-PAMI}. In applying it to a particular scenario, however, at least three critical questions must be answered:

First, {\em what is the most appropriate class of low-dimensional models $\mc M_i$?} The proper class of models may depend on the properties of the object $\mc O$, as well as the types of nusiance variations that may be encountered. For example, variations in illumination may be well-captured using low-dimensional {\em linear} models \cite{Georghiades2001-PAMI,Basri2003-PAMI}, whereas variations in pose or alignment are highly nonlinear \cite{Donoho2005-JMIV}. 

Second, {\em  how should we measure the distance $d(\mb q,\mc M_i)$ between $\mb q$ and $\mc M_i$?} Typically, one adopts a metric $\mathrm{dist}\left(\cdot, \cdot\right)$ on $\Re^D$, and then sets 
\[
d(\mb q,\mc M_i) = \min_{\mb v \in \mc M_i} \mathrm{dist}\left(\mb q, \mb v\right). 
\]
Here, again, the appropriate metric $\mathrm{dist}\left(\cdot, \cdot\right)$ depends on our prior knowledge. For example, if the observation $\mb q$ is known to be perturbed by iid Gaussian noise, minimizing the metric induced by the $\ell^2$ norm $\mathrm{dist}\left(\mb q, \mb v\right) = \| \mb q - \mb v \|_2$ yields a maximum likelihood estimator. However, in practice other norms may be more appropriate: for example, in situations where the data may have errors due to occlusions, shadows, specularities, the $\ell^1$ norm is a more robust alternative \cite{Wright2009-PAMI}. 

Finally, given an appropriate model and error distance, {\em how can we efficiently determine the nearest model to a given input query?} That is to say, we would like to solve 
\begin{equation}
\mathop{\arg\min}_{i \in \left\{1, \cdots, n\right\}} \min_{\mb v \in \mc M_i} \mathrm{dist}\left(\mb q, \mb v\right)
\end{equation}
using computational resources that depend as gracefully as possible on the ambient dimension $D$ (typically number of pixels in the image) and the number of models $n$. In practical applications, both of these quantities could be very large.

\paragraph{This paper} In this paper, we consider the case when the low-dimensional models $\mc M_i$ are {\em linear subspaces}. As mentioned above, subspace models are well-justified for modeling illumination variations \cite{Georghiades2001-PAMI,Basri2003-PAMI} (say, in near-frontal face recognition), and also form a basic building block for modeling and computing with more general, nonlinear sets \cite{Simard98transformationinvariance,Roweis2000-Science}. 

Our methodology pertains to distances $\mathrm{dist}(\mb q,\mb v)$ induced by the $\ell^p$ norm $\| \mb q - \mb v \|_p$, with $p \in (0,2]$\footnote{Mathematically $\norm{\mb x}{p} = \left(\sum_{i} \left|x_i\right|^p\right)^{1/p}$ defines a valid norm only when $p \geq 1$, which in turn induces valid metric $\norm{\mb x - \mb y}{p}$. For $p \in \left(0, 1\right)$, though $\norm{\cdot}{p}$ is not a valid norm, one can verify that $\norm{\mb x}{p}^p = \sum_{i} \left|x_i\right|^p$ indeed also induces valid metric, i.e., for all $\mb x, \mb y, \mb z \in \R^D$, $\norm{\mb x - \mb y}{p}^p \geq 0$, $\norm{\mb x - \mb y}{p}^p = 0 \Longleftrightarrow \mb x = \mb y$, $\norm{\mb x - \mb y}{p}^p = \norm{\mb y - \mb x}{p}^p$, and also the triangular inequality holds: 
$\norm{\mb x - \mb z}{p}^p \leq \norm{\mb x - \mb y}{p}^p + \norm{\mb y - \mb z}{p}^p$. These latter cases may turn out to be empirically interesting, as $\ell^p$ ``norm" for $p \in \left(0, 1\right)$ is actually sharper proxy for the $\ell^0$ counting norm (which is the main count for robustness to errors as discussed in subsequent parts) than the $\ell^1$ norm. Since stable distributions exist for all $\norm{\cdot}{p}$ ($p \in \left(0, 2\right]$), our current algorithm and analysis methodology is likely to extend to all $p \in \left(0, 2\right]$.}. We focus here on the $\ell^1$ norm, $\| \mb q - \mb v \|_1 = \sum_i | q_i - v_i |$. The $\ell^1$ norm is a natural and well-justified choice when the test image contains pixels that do not fit the model -- say, due to moderate occlusion, cast shadows, or specularities \cite{Wright2009-PAMI}. For $p \in (0,2]$, the $\ell^p$ norm with $p = 1$ strikes a unique compromise between computational tractability (convexity) and robustness to gross errors. 

With this choice of models and distance, at recognition time we are left with the following computational task:
\begin{problem}\label{prob:main}
Given $n$ linear subspaces $\mc S_1, \dots, \mc S_n$ of dimension $r$ and a query point $\mb q$, all in $\Re^D$, determine the nearest $\mc S_i$ to $\mb q$ in $\ell^1$ norm. 
\end{problem}

This problem has a straightforward solution: solve a sequence of $n$ $\ell^1$ regression problems:
\begin{equation} \label{eqn:L1}
  \min_{\mb v \in \mc S_i} \| \mb q - \mb v\|_1,
\end{equation}
and choose the $i$ with the smallest optimal objective value. The total cost is $O( n \cdot T_{\ell^1}(D,r) )$, where $T_{\ell^1}(D,r)$ is the time required to solve the linear program \eqref{eqn:L1}. For example, for interior point methods \cite{Boyd}, we have $T_{\ell^1}(D,r) = O(D^{3.5})$ \footnote{We have suppressed the dependency on other factors, such as $\log \frac{1}{\eps}$ (where $\eps$ denotes the target precision) and $r$ to make things concise, because our main interest is mostly in the effect of $D$ on the complexity. Lower order is possible for our specific case by some careful implementation, see, e.g., 11.8.2, page 617 of~\cite{Boyd}. See also our discussion of running time in Section~\ref{sec:running_time}. }. There exist more scalable first-order methods \cite{Efron04leastangle,Beck2009-SJIS,Yin_bregmaniterative,Yang2010-ICIP}, which improve on the dependence on $D$ at the expense of higher iteration complexity. The best known complexity guarantees for each of these methods are again superlinear in $D$, although linear runtimes may be achievable when the residual $\mb q - \mb v_\star$ is very sparse \cite{Donoho06fastsolution} or the problem is otherwise well-structured \cite{AgarwalNIPS-2011}. Even in the best case, however, the aforementioned algorithms have complexity $\Omega( n D )$.\footnote{On a more technical level, when the $\mc S_i$ are fit to sample data, the aforementioned first-order methods may require tuning for optimal performance.} 
When both terms are large, this dependence is prohibitive: Although Problem \ref{prob:main} is simple to state and easy to solve in polynomial time, achieving real-time performance or scaling massive databases of objects appears to require a more careful study.

In this paper, we present a very simple, practical approach to Problem \ref{prob:main}, with much improved computational complexity, and reasonably strong theoretical guarantees. Rather than working directly in the high-dimensional space $\Re^D$, we randomly embed the query $\mb q$ and subspaces $\mc S_i$ into $\Re^d$, with $d \ll D$. The random embedding is given by a $d \times D$ matrix $\mb P$ whose entries are i.i.d. standard Cauchy random variables. That is to say, instead of solving \eqref{eqn:L1}, we solve
\begin{equation} \label{eqn:L1-proj}
\min_{\mb v \in \mc S_i} \| \mb P \mb q - \mb P \mb v \|_1.
\end{equation}
We prove that if the embedded dimension $d$ is sufficiently large -- say $d = \mathrm{poly}( r \log n )$ (i.e., $d$ bounded by some polynomial of $r\log n$), then with constant probability the model $\mc S_i$ obtained from \eqref{eqn:L1-proj} is the same as the one obtained from the original optimization \eqref{eqn:L1}. 

The required dimension $d$ does not depend in any way on the ambient dimension $D$, and is often significantly smaller: e.g., $d = 25$ vs.\ $D = 32,000$ for one typical example of face recognition. The resulting (small) $\ell^1$ regression problems can be solved very efficiently using customized interior point solvers (e.g., \cite{Mattingley}). These methods are numerically reliable, and can yield a speedup of several folds over the standard approach relying on solving \eqref{eqn:L1}. 

The price paid for this improved computational profile is a small increase in the probability of failure of the recognition algorithm, due to the use of a randomized embedding. Our theory quantifies how large $d$ needs to be to render this probability of error under control. Repeated trials with independent projections $\mb P$ can then be used to make the probability of failure as small as desired. Because $\ell^1$ regression is so much cheaper in the low-dimensional space $\Re^d$ than in the original space $\Re^D$ provided $d \ll D$, these repeated trials are affordable. 

The end result is a simple, practical algorithm that guarantees to maintain the good properties of $\ell^1$ regression, with substantially improved computational complexity. We demonstrate this on model problems in subspace-based face and object instance recognition. In addition to improved complexity in theory, we observe remarkable improvements on real data examples, suggesting that point-to-subspace query in $\ell^1$ could become a  practical strategy (or basic building block) for face and object recognition tasks involving large databases, or small databases under hard time constraints. 

\paragraph{Relationship to existing work} Problem \ref{prob:main} is an example of a {\em subspace search} problem. For $0$-dimensional affine subspaces in $\ell^2$ (i.e., points), this problem coincides with the nearest neighbor problem. Its approximate version can be solved in time {\em sublinear} in $n$, the number of points, using randomized techniques such as locality sensitive hashing \cite{Datar04LSH}. When the dimension $r$ is larger than zero, the problem becomes significantly more challenging. For the case of $r=1$, sublinear time algorithms exist, although they are more complicated \cite{Andoni_approximateline}. 

Recently two groups have proposed approaches to tackling larger $r$. Basri et.\ al.\ \cite{Basri2011-pami} lift subspaces into a higher dimensional vector space (identifying the subspace with its $D\times D$ orthoprojector) and then apply point-based near neighbor search. Jain et.\ al.\ give several random hash functions for the case when the $\mc S_i$ are hyperplanes \cite{Jain2010-NIPS}. Both of these approaches pertain to $\ell^2$ only. Both perform well on numerical examples, but have limitations in theory, as neither is known to yield an algorithm with provably sublinear complexity for all inputs. Results in theoretical computer science suggest that these limitations may be intrinsic to the problem: a sublinear time algorithm for approximate nearest hyperplane search would refute the strong version of the ``exponential time hypothesis'', which conjectures that general boolean satisfiability problems cannot be solved in time $O(2^{c n})$ for any $c < 1$ \cite{Williams05anew}.

The above algorithms exploit special properties of the $\ell^2$ version of Problem \ref{prob:main}, and do not apply to its $\ell^1$ variant. However, the $\ell^1$ variant retains the aforementioned difficulties, suggesting that an algorithm for $\ell^1$ near subspace search with sublinear dependence on $n$ is unlikely as well.\footnote{Although it could be possible if we are willing to accept time and space complexity exponential in $r$ or $D$, ala \cite{Magen2008}.} This motivates us to focus on ameliorating the dependence on $D$. Our approach is very simple and very natural: Cauchy projections are chosen because the Cauchy family is the unique $\ell^1$-stable distribution, i.e., Cauchy projection of any given vector remains iid Cauchy (see Equation~\eqref{eq:l1_stable_def} and Appendix~\ref{app:preliminary} for details), a property which has been widely exploited in previous algorithmic work \cite{Datar04LSH,Li2007-JMLR,sohler2011subspace}. 

However, on a technical level, it is not obvious that Cauchy embedding should succeed for this problem. The Cauchy is a heavy tailed distribution, and because of this it does not yield embeddings that very tightly preserve distances between points, as in the Johnson-Lindenstrauss lemma\footnote{One version of the lemma (taken from~\cite{dasgupta2003elementary}) states that: for any $\eps \in \left(0, 1\right)$ and any $n \in \N$, let $k \in \N$ satisfy $k \geq 4\left(\eps^2/2 - \eps^3/3\right)^{-1} \log n$. Then for any set $\mc V$ of $n$ points in $\R^d$, there is a map $f: \R^d \to \R^k$ such that for all $\mb u, \mb v \in \mc V$, $\left(1-\eps\right) \norm{\mb u - \mb v}{}^2 \leq \norm{f\left(\mb u\right) - f\left(\mb v\right)}{}^2 \leq \left(1+\eps\right)\norm{\mb u -\mb v}{}^2$. Note in particular that $k$ is independent of the ambient dimension $d$, and depends on $n$ only through its logarithm. } (JL Lemma, ~\cite{johnson1984extensions, dasgupta2003elementary}). In fact, for $\ell^1$, there exist lower bounds showing that certain point sets in $\ell^1$ cannot be embedded in significantly lower-dimensional spaces without incurring non-negligible distortion \cite{Brinkman} \footnote{In particular, it is shown in~\cite{Brinkman} that to keep the distortion within $\eps$, it is necessary the projection dimension is $n^{\Omega\left(1/\eps^2\right)}$. }. For a single subspace, embedding results exist -- most notably due to Sohler and Woodruff \cite{sohler2011subspace}, but the distortion incurred is so large as to render them inapplicable to Problem \ref{prob:main}. Nevertheless, several elegant technical ideas in the proof of \cite{sohler2011subspace} turn out to be useful for analyzing Problem \ref{prob:main} as well. 

The problem studied here is also related to recent work on sparse modeling and sparse error correction. Indeed, one of the strongest technical motivations for using the $\ell^1$ norm is its provable good performance in sparse error correction \cite{CandesE2005-IT,Wright2008-IT}. These results give conditions under which it is possible to recover a vector $\mb v$ from grossly corrupted observation
\[
\mb q = \mb v + \mb e,
\] 
with $\mb v \in \mc S$, and the sparse error $\mb e$ unknown. These results are quite strong: they imply exact recovery, even if the error $\mb e$ has constant fractions of nonzero entries, of arbitrary magnitude. For example, \cite{CandesE2005-IT} proves that under technical conditions, $\ell^1$ minimization 
\begin{equation}\label{eqn:L1-r}
\min \| \mb e \|_1 \quad \text{s.t.} \quad \mb q - \mb e \in \mc S
\end{equation}
 exactly recovers $\mb e$ when $\mc S$ is a linear subspace. \cite{Wright2008-IT} presents similar theory for the case when $\mc S$ is a union of linear subspaces solved by a variant of optimization in \eqref{eqn:L1-r}. 

On the other hand, exact recovery may be stronger than what is needed for recognition. For recognition, as formulated in this work, we only need to know which subspace minimizes the distance $d(\mb q,\mc S_i)$ -- we do not need to precisely estimate the difference vector itself. The distinction is important: while \cite{Wright2009-PAMI} shows that significant dimensionality reduction is possible if there are no gross errors $\mb e$, when errors are present, the cardinality of the error vector gives a hard lower bound on the number of observations required for correct recovery. In contrast, for the simpler problem of finding the nearest model, it is possible to give an algorithm that uses very small $d$, and is agnostic to the properties of $\mb q$ and $\mc S_1 \dots \mc S_n$. 

To solve the component regression problem in projected space is also reminiscent of research on approximate $\ell^1$ regression (see, e.g., ~\cite{sohler2011subspace, ClarksonDrineasEtAl2012Fast}). The purpose in that line of work is to efficiently obtain an $\eps$-approximate solution to a single $\ell^1$ regression: any $\mb x$ such that 
\[
\norm{\mb y - \mb A\mb x}{\ell^1} \leq \left(1 + \eps\right) \min_{\mb z}\norm{\mb y - \mb A\mb z}{\ell^1}. 
\]
Our purpose here is quite different: for a bunch of $\ell^1$ regression  problems, instead of being concerned with quality of solving each individual problem, one only needs to ensure that the regression problem with the smallest objective value remains so after approximation. Moreover, state-of-the-art coreset-based approximation algorithms for $\ell^1$ regression such as those in~\cite{dasgupta2008sampling, sohler2011subspace, ClarksonDrineasEtAl2012Fast} depend heavily on obtaining some importance sampling measure (e.g., $\ell^1$ leverage score of an $\ell^1$ well conditioned basis in~\cite{ClarksonDrineasEtAl2012Fast}), which in turn depends on $\mb A$ and $\mb y$ simultaneously. In a database-query model that is common in recognition tasks, this complicated dependency directs lots of computation to query time. By comparison, considerable portion of computation (e.g., projection of the subspaces) in our framework can be performed during training, rendering the framework attractive when the recognition is under hard time constraint. 

\paragraph{Notation} We define some most commonly used notations here. $d_{\ell^1}\left(\cdot, \cdot\right)$ is the $\ell^1$ distance of a point to a subspace, i.e., $d_{\ell^1}\left(\mb q, \mc S\right) = \min_{\mb v \in \mc S} \norm{\mb q - \mb v}{\ell^1}$. For any $k \in \N$, $[k] = \left\{1, \cdots, k\right\}$ and $\equiv_d$ denotes equality in distribution. Other notations will be defined inline. 

\section{Our Algorithm and Main Results} \label{sec:overview}
The flow of our algorithm is summarized as follows. 
\begin{table}[!htbp]
\centering
  \begin{tabular}{p{0.9\linewidth}}
  \hline
      \textbf{Input}: $n$ subspaces $\mc S_1, \cdots, \mc S_n$ of dimension $r$ and query $\mb q$ \\
      \textbf{Output}: Identity of the closest subspace $\mc S_\star$ to $\mb q$ \\
      \hline \hline 
      \textbf{Preprocessing}: Generate $\mb P \in \R^{d\times D}$ with iid Cauchy RV's ($d \ll D$) and compute the projections $\mb P\mc S_1$, $\cdots$, $\mb P \mc S_n$; Repeat for independent repetitions of $\mb P$ \\
 \\
       \textbf{Candidates Search}: Compute the projection $\mb P \mb q$, and compute its $\ell^1$ distance to each of $\mb P \mc S_i$. Repeat for several versions of $\mb P$, and locate nearest candidates \\
       \\
       \textbf{Refined Scanning}: Scan the candidates in $\R^D$ and return $\mc S_\star$.  \\
\hline
  \end{tabular}
\end{table}
\\
Our main theoretical result states that if $d$ is chosen appropriately, with at least constant probability, the subspace $\mc S_{i_\star}$ selected will be the original closest subspace $\mc S_\star$:
\begin{theorem}\label{thm:main}
Suppose we are given $n$ linear subspaces $\left\{\mc S_1, \cdots, \mc S_n\right\}$ of dimension $r$ in $\R^D$ and any query point $\mb q$, and $d_{\ell^1}\left(\mb q, \mc S_1\right) \leq d_{\ell^1}\left(\mb q, \mc S_i\right)/\eta$ for all $i\in [n]\setminus\left\{1\right\}$ and some $\eta > 1$. Then for any fixed $\alpha < 1 - 1/\eta$, there exists $d = O\left[\left(r \log n\right)^{1/\alpha}\right]$ (assuming $n > r$), such that if $\mb P  \in \R^{d\times D}$ is iid Cauchy, we have
\begin{equation}
\mathop{\arg\min}_{i\in [n]} d_{\ell^1}\left(\mb P \mb{q}, \mb P \mc S_i\right) = 1
\end{equation}
 with (nonzero) constant probability. 
\end{theorem}

The choice of the first subspace as the nearest is only for notational and expository convenience. Also we write $\mathop{\arg\min}_{i\in \left[n\right]} d_{\ell^1}\left(\mb P \mb q, \mb P \mb S_i\right) = 1$ to mean that the first subspace is the nearest \emph{unambiguously}, i.e., the set of minimizers is a singleton (this comment applies to similar situations below). The condition in Theorem \ref{thm:main} depends on several factors. Perhaps the most interesting is the relative gap $\eta$ between the closest subspace distance and the second closest subspace distance. Notice that $\eta \in [1,\infty)$, and that the exponent $1/\alpha$ becomes large as $\eta$ approaches one. This suggests that our dimensionality reduction will be most effective when the relative gap is nonnegligible. For example, when $\eta = 2$ the required dimension is proportional to $r^2$. 

Notice also that $d$ depends on the number of models $n$ only through its logarithm. This rather weak dependence is a strong point, and, interestingly, mirrors the Johnson-Lindenstrauss lemma for dimensionality reduction in $\ell^2$, even though JL-syle embeddings are impossible for $\ell^1$. 

Before stating our overall algorithm, we suggest two additional practical implications of Theorem \ref{thm:main}. First, Theorem \ref{thm:main} only guarantees success with constant probability. This probability is easily amplified by taking $T$ independent trials. Because the probability of failure drops exponentially in $T$, it usually suffices to keep $T$ rather small. Each of these $T$ trials generates one or more candidate subspaces $\mb S_i$. We can then perform $\ell^1$ regression in $\Re^D$ to determine which of these candidates is actually nearest to the query. Note that it may also be possible to perform this second step in $\Re^{d'}$, where $d < d' \ll D$. 

Second, the importance of the gap $\eta$ suggests another means of controlling the resources demanded by the algorithm. Namely, if we have reason to believe that $\eta$ will be especially small (i.e., approaching one), we may instead set $d$ according to the gap between $\xi_{1'}$ and $\xi_{k'}$, for some $k' > 2$, where for any $i \in [n]$, $\xi_{i'}$ denotes the $\ell^1$ distance of the query $\mb q$ to its $i^{th}$ nearest subspace. With this choice, Theorem \ref{thm:main} implies that with constant probability the desired subspace is amongst the $k' - 1$ nearest to the query. Again, all of these $k'-1$ subspaces need to be retained for further examination. However, if $k' \ll n$, this is still a significant saving over the standard approach. 

We complement our main result above with a result on the lower bound of the projecting dimension $d$, which basically says any randomized embedding that is oblivious to the query and subspaces has the target dimension dictated by $\log n, r$ and reciprocal of $\log \eta_{\min}$, where $\eta_{\min}$ is a nominal relative distance gap (see below), in order to preserve the identity of the nearest subspace with non-negligible probability. 

\begin{theorem} \label{thm:lower_bound}
Fix any $r, n\in \N, \eta_{min} \in \left(1, \infty\right)$ and $\gamma \in \left(1/n, 1\right)$. Let $d \in \N$ satisfy: for all $\; D \geq r$, there exists a distribution $\mu$ over $\R^{d \times D}$, such that for all set $\left\{\mc S_1, \cdots, \mc S_n\right\}$ of $r$-dimensional subspaces and point $\mb q$ in $\R^D$ with the property $d_{\ell^1}\left(\mb q, \mc S_1\right) \leq d_{\ell^1}\left(\mb q, \mc S_i\right)/\eta_{\min}$ for all $i \in[n]$, one has 
\begin{equation}
\mbb{P}_{\mb P \sim \mu}\left[\mathop{\arg\min}_{i\in \left[n\right]} d_{\ell^1}\left(\mb P \mb q, \mb P \mc S_i\right) = 1\right] \geq \gamma. 
\end{equation}
Then $d \geq \max\left(C_1 \frac{1}{\log 3\left(\eta_{\min} + 1\right)}\log\frac{1}{1-\gamma}\log n - C_2 \frac{r}{\log r}, r\right)$ for some numerical constants $C_1, C_2$. 
\end{theorem}

We restrict the probability to be greater than $1/n$ to rule out any case worse than random guess.  The proof is provided in Appendix~\ref{app:proof_lower_bound}. We note that there is a significant gap between the upper bound in Theorem~\ref{thm:main} and the lower bound in Theorem~\ref{thm:lower_bound}. In particular, it is not clear whether $\eta_{\min}$ should enter the bound in its current form, which is extremely bad for small $\eta_{\min}$, or resemble our lower bound, which is significantly milder. To resolve these issues remains an open problem.  

\section{A Sketch of the Analysis} \label{sec:sketch_proof}

In this section, we sketch the analysis leading to Theorem \ref{thm:main}. The basic rationale for using Cauchy projection is that the standard Cauchy is a {\em stable} distribution for the $\ell^1$ norm: if $\mb v \in \Re^D$ is any fixed vector, and $\mb P \in \Re^{d \times D}$ is a matrix with iid Cauchy entries, then the vector 
\begin{align} \label{eq:l1_stable_def}
\mb P \mb v \equiv_d \| \mb v \|_1 \times \mb z,
\end{align}
where $\mb z$ is again an iid Cauchy vector. In fact, the Cauchy family is also the only stable distribution for the $\ell^1$ norm (see Appendix~\ref{app:preliminary} for more details). So, $\| \mb P \mb v \|_1 \equiv_d \| \mb v \|_1 \| \mb z \|_1 = \| \mb v \|_1 \sum_i |z_i|$. The random variables $|z_i|$ are iid {\em half-Cauchy}, with probability density function 
\begin{equation}
f_{\mc{HC}}(x) = \frac{2}{\pi} \frac{1}{1+x^2} \;\quad \text{if}\;\; x \ge 0, 
\end{equation}
and $f_{\mc{HC}}(x) = 0$ for $x < 0$. 

In point-to-subspace query, we need to understand how $\mb P$ acts on many vectors $\mb v$ simultaneously -- including the query $\mb q$ and all of the subspaces $\mc S_1 \dots \mc S_n$. Here, we encounter a challenge: although the Cauchy is unambiguously the correct distribution for estimating $\ell^1$ norms, it is rather ill-behaved: its mean and variance do not exist, and the sample averages $\tfrac{1}{n} \sum_i |z_i|$ do not obey the classical Central Limit Theorem. 

\newcommand{\done}[2]{d_{\ell^1}\left(#1,#2\right)}
\newcommand{\bests}{\mc S_\star}

Fig. \ref{fig:distances} shows how this behavior affects the point-to-subspace distance $\done{\mb q}{\mc S}$. The figure shows a histogram of the random variable $\psi = \done{\mb P \mb q}{\mb P \mc S}$, over randomly generated Cauchy matrices $\mb P$, for two different configurations of query $\mb q$ and subspace $\mc S$. Two properties are especially noteworthy. First, the upper tail of the distribution can be quite heavy: with non-negligible probability, $\psi$ may significantly exceed its median. On the other hand, the lower tail is much better behaved: with very high probability, $\psi$ is not significantly smaller than its median. 
\begin{figure}[!htbp]
\centering
\includegraphics[width = 0.65\linewidth]{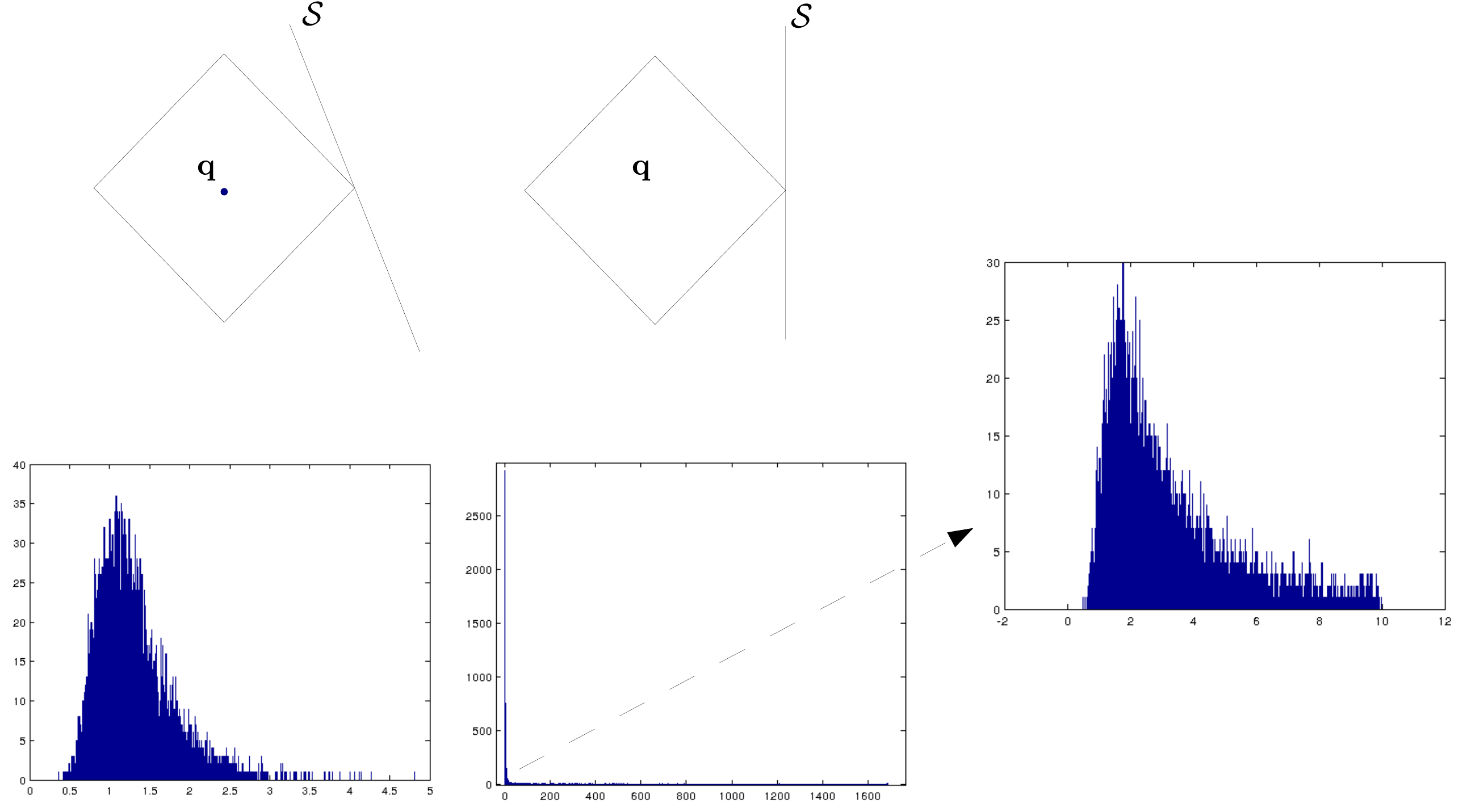}
\caption{Statistics of $\ell^1$ distance ratios (after vs. before) by random projections over $10000$ trials. The subspaces are randomly-oriented ($1^{st}$ column) and axis-aligned ($2^{nd}$ column), respectively. Here $r=10$, $D = 10000$, $d = 35$, and $d_{\ell^1}\left(q, \mc S\right) = 1$.  } \label{fig:distances}
\end{figure}
This inhomogeneous behavior (in particular, the heavy upper tail) precludes very tight distance-preserving embeddings using the Cauchy. However, our goal is {\em not} to find an embedding of the data, per se, but rather to find the nearest subspace, $\bests$, to the query. In fact, for nearest subspace search, this inhomogeneous behavior is much less of an obstacle. To guarantee to find $\bests$, we need to ensure qualitatively that 
\begin{quote}
\begin{itemize} 
\item[- (i) ] $\mb P$ does not increase the distance from $\mb q$ to $\bests$ too much, and, 
\item[- (ii)] $\mb P$ does not shrink the distance from $\mb q$ to any of the other subspaces $\mc S_i$ too much. 
\end{itemize}
\end{quote}
The first property, (i), holds with constant probability: although the tail of $\psi$ is heavy, with probability at least $1/2$, $\psi \le \mathrm{median}(\psi)$. For the second event, (ii), $\mb P$ needs to be well-behaved on $n-1$ subspaces simultaneously. Notice, however, that for the bad subspaces $\mc S_i$, the lower tail in Figure \ref{fig:distances} is most important. If projection happens to significantly increase the distance between $\mb q$ and $\mc S_i$, this will not cause an error (and may even help, in the sense that amplifying the distance to a ``bad'' subspace renders the event that the ``good'' subspace be mis-detected (hence failure) less likely). Since the lower tail is sharp, we {\em can} guarantee that if $d$ is chosen correctly, $\mb P \mb q$ will not be significantly closer to any of the $\mb P \mc S_i$. 

Below we describe some of the technical manipulations needed to carry this argument through rigorously, and state key lemmas for each part. Sec. \ref{sec:no-expand} elaborates on property (i), while Sec. \ref{sec:no-contract} describes the arguments needed to establish property (ii). Theorem \ref{thm:main} follows directly from the results in Secs. \ref{sec:no-expand} and \ref{sec:no-contract}. This argument, as well as proofs of several routine or technical lemmas are deferred to the appendix. 

\subsection{Bounded expansion for the good subspace} \label{sec:no-expand} Let $\mb v_\star \in \mc S_\star$ be a closest point to $\mb q$ in $\ell^1$ norm, before projection:
\[
\mb v_\star \in \arg \min_{\mb v \in \mc S_\star} \| \mb q - \mb v \|_1.
\]
Such a point $\mb v_\star$ may not be unique, but always exists. After projection, $\mb P \mb v_\star$ might no longer be the closest point to $\mb P \mb q$. However, the distance $\| \mb P \mb q - \mb P \mb v_\star \|_1$ {\em does} upper bound the distance from $\mb P \mb q$ to $\mb P \mc S_\star$:
$$\done{\mb P \mb q}{\mb P \mc S_\star} \;=\; \min_{\mb h \in \mb P \mc S_\star} \| \mb P \mb q - \mb h \|_1 \;\le\; \| \mb P \mb q - \mb P \mb v_\star\|_1 \;=\; \| \mb P (\mb q- \mb v_\star) \|_1.$$
Hence, it is enough to show that $\mb P$ preserves the norm of the particular vector $\mb w = \mb q - \mb v_\star$. We use the following lemma for this purpose, the proof of which can be found in Appendix~\ref{app:proof_expansion}. 
\begin{lemma} \label{lemma:good_subspace}
There exists a numerical constant $c \in \left(0, 1\right)$ with the following property. If $\mb w \in \Re^D$ be any fixed vector, $2 \leq d \in \N$, and suppose that $\mb P \in \Re^{d \times D}$ is a matrix with i.i.d. standard Cauchy entries, then 
\begin{equation}
\prob{\| \mb P \mb w \|_1 \leq \frac{2}{\pi} d \log d \, \| \mb w \|_1 } \;\geq \; c. 
\end{equation}
\end{lemma}

\subsection{Bounded contraction for the bad subspaces} \label{sec:no-contract} For the ``bad'' subspaces $\mc S_2 \dots \mc S_n$, our task is more complicated, since we have to show that under projection $\mb P$, {\em no} point in $\mc S_i$ comes close to $\mb q$. In fact, we will show something slightly stronger: for appropriate $\gamma$, with high probability the following holds for any $i$:
\begin{equation} \label{eqn:sum-contract}
\forall \;\mb w \in \mc S_i \oplus \mathrm{span}( \mb q ), \quad \| \mb P \mb w \|_1 \;\ge\; \gamma \| \mb w \|_1. 
\end{equation}
Above, $\oplus$ denotes the direct sum of subspaces, so $\tilde{\mc S}_i = \mc S_i \oplus \mathrm{span}(\mb q)$ is the linear span of $\mc S_i$ and the query together. Since for any $\mb v \in \mc S_i$, $\mb q - \mb v \in \tilde{\mc S}_i$, whenever \eqref{eqn:sum-contract} holds, we have 
\begin{eqnarray}
\done{\mb P \mb q}{\mb P \mc S_i} \;&=&\; \min_{\mb v \in \mc S_i} \| \mb P \mb q - \mb P \mb v \|_1 \quad\ge\quad \min_{\mb v \in \mc S_i} \| \mb P (\mb q - \mb v) \|_1 \nonumber \\ &\ge&\; \min_{\mb v \in \mc S_i} \gamma \| \mb q - \mb v \|_1 \quad=\quad \gamma \, \done{\mb q}{\mc S_i},
\end{eqnarray}
and the distance to any ``bad'' subspace $\mc S_i$ contracts by at most a factor of $\gamma$. 

To show \eqref{eqn:sum-contract}, we use a discretization argument. Let $\Gamma$ denote the intersection of the unit $\ell^1$ ``sphere'' with the expanded subspace $\tilde{\mc S}_i$: 
$$\Gamma = \{ \mb w \mid \| \mb w \|_1 = 1 \} \cap \tilde{\mc S}_i.$$
Recall that for any set $\Gamma$, an {\em $\eps$-net} is a subset $N_i \subset \Gamma$ such that for every $\mb w \in \Gamma$, $\| \mb w - \mb w' \|_1 \le \epsilon$ for some $\mb w' \in N_i$. Standard arguments (see Lemma $3.18$, page $63$ of \cite{Ledoux}) show that for any $\epsilon > 0$, there exists an $\epsilon$ net $N_i$ for $\Gamma$ of size at most $(3/\epsilon)^{r+1}$.

Consider the following two events:
\begin{quote}
\begin{itemize}
\item [- (ii.a)] $\min_{\mb w' \in N} \; \| \mb P \mb w' \|_1 \ge \beta$, and 
\item [- (ii.b)] For all $\mb w \in \tilde{\mc S}_i$, $\| \mb P \mb w \|_1 \;\le\; L \| \mb w\|_1$.
\end{itemize}
\end{quote}
When both hold, we have for any $\mb w \in \Gamma$ (with associated closest point $\mb w' \in N_i$)
\begin{eqnarray}
\| \mb P \mb w \|_1 &\ge& \| \mb P \mb w' + \mb P (\mb w - \mb w') \|_1 \;\ge\; \| \mb P \mb w' \|_1 - \| \mb P (\mb w - \mb w' ) \|_1 \;\ge\; \beta - L \epsilon.
\end{eqnarray}
Moreover, since for any $\mb w \in \tilde{\mc S}_i$, $\mb w / \| \mb w \|_1 \in \Gamma$, we have that $$\forall \; \mb w \in \tilde{\mc S}_i, \quad \|\mb P \mb w \|_1 \ge ( \beta - L \epsilon ) \| \mb w \|_1,$$
and we may set $\gamma = \beta - L \epsilon$. So, it is left to establish items (ii.a) and (ii.b) above. 

\paragraph{Establishing (ii.a)} We use the following tail bound:
\begin{lemma}[Concentration in Lower Tail]\label{lemma:contract}
Let $\mb P \in \Re^{d \times D}$ be an iid Cauchy matrix. Then for any fixed vector $\mb w \in \R^D$ and $\alpha, \delta \in (0,1)$, 
\begin{equation}
\prob{\norm{\mb P \mb w}{1} < \left(1-\alpha\right)\left(1-\delta\right) \frac{2}{\pi} d\log d\norm{\mb w}{1}} \;<\; d^{1-\alpha}\exp\left(-\frac{\delta^2}{2\pi}d^{\alpha}\right). 
\end{equation}
\end{lemma}
In hindsight, the exponent $\alpha$ in the power gives rise to the exponential factor in our bound for $d$ in Theorem~\ref{thm:main}. Unfortunately, we are able to establish a concrete lower bound on the probability, which shows this estimate gives the optimal power. Detailed discussions and proofs are deferred to Appendix~\ref{app:proof_contract}. 

This bound is sharp enough to allow us to simultaneously lower bound $\| \mb P \mb w' \|_1$ over all $\mb w' \in N_i$. Set $$\beta_{\alpha,\delta} = (1-\alpha)(1-\delta) \tfrac{2}{\pi} d \log d,$$ and let $\event_{\text{net},i}$ denote the event that there exists $\mb w' \in N_i$ with $\| \mb P \mb w' \|_1 < \beta_{\alpha,\delta} \| \mb w' \|_1$. 
\begin{equation}
\prob{ \event_{\text{net},i} } \;\;<\;\; |N_i| \, d^{1-\alpha} \exp\left(- \tfrac{\delta^2}{2\pi} d^\alpha \right).
\end{equation}

\paragraph{Establishing (ii.b)} In bounding the Lipschitz constant $L$ in (ii.b), we have to cope with the heavy tails of the Cauchy, and simple arguments like the above argument for $\beta$ are insufficient. Rather, we borrow an elegant argument of Sohler and Woodruff \cite{sohler2011subspace}. The rough idea is to work with a certain special basis for $\tilde{\mc S}_i$, which can be considered an $\ell^1$ analogue of an orthonormal basis. Just as an orthonormal basis preserves the $\ell^2$ norm, an {\em $\ell^1$ well-conditioned basis} approximately preserves the $\ell^1$ norm, up to distortion $(r+1)$. The argument then controls the action of $\mb P$ on the elements of this basis. Due to space limitations, we defer further discussion of this idea to Appendix~\ref{app:proof_lips}, and instead simply state the resulting bound:
\begin{lemma} \label{lemma:lipschitz}
Let $\mb P \in \Re^{d \times D}$ be an iid Cauchy matrix, and $\mc S$ a fixed subspace of dimension $r+1$. Set $L = \sup_{\mb w \in \mc S \setminus \{ \mb 0 \}} \| \mb P \mb w \|_1/\| \mb w \|_1$. Then for any $B > 0$, we have 
\begin{equation}
\prob{ L > t\left(r+1\right) } \le \frac{2 d (r+1)}{\pi B} + \frac{2 d}{\pi t} \log \sqrt{1+B^2}. 
\end{equation}
\end{lemma}

The proof of Theorem \ref{thm:main} follows from Lemmas 1-3 above, by choosing appropriate values of the parameters $B$, $t$, $\delta$ and $\epsilon$. We give the detailed calculation in Appendix~\ref{app:proof_main}. 
\begin{remark}
We do not allow $\eta = 1$ in Theorem~\ref{thm:main}, corresponding to ties in the nearest subspaces. In this special case, it seems natural that one instead ask the dimension reduction to preserve \emph{any one of the nearest subspaces}; the problem actually becomes easier. To see this, one can fix one of the nearest subspaces as the ``good'' one, ignore the rest of the nearest, and treat all the rest as ``bad'' subspaces. Now the new relative distance gap $\eta_{\mathrm{effective}} > 1$, and the number of distances we want to control becomes smaller than the number of subspaces present, hence the problem is actually \emph{easier} as compared to a generic problem setting as in Theorem~\ref{thm:main} with the same parameters (except for the slightly slacked target as stated above).
\end{remark}

\section{Experiments}
We present three experiments to corroborate our theoretical results and demonstrate their particular relevance to subspace-based robust instance recognition. 
\subsection{Note on Implementation}
\paragraph{Projection Matrices and Subspaces} Theorem~\ref{thm:main} is for any fixed set of subspaces and any fixed query point. Of course, if we fix the projection matrix $\mb P$ and consider many different query points, the success or failure of approximation to each query will be dependent. This suggests sampling a new matrix $\mb P$ for each new query, which would then require that we re-project each of the subspaces $\left\{\mc S_i\right\}$. In practice, it is more efficient to maintain a pool of $k$ Cauchy projection matrices\footnote{The standard Cauchy projection matrix $\mb P$ can be generated as $\mb A./\mb B$, where both $\mb A$ and $\mb B$ are i.i.d. standard normals and ``$./$'' denotes element-wise matrix division. } $\left\{\mb P_j\right\}$ and store $\mb P_j \mc S_i$ for each $i$ and $j$. During testing, we randomly sample a combination of $N_{rep}$ (``rep'' for repetition) matrices and corresponding projected subspaces and also apply these projections to the query. This sampling strategy from a finite pool does not generate independent projections for different query points, but it allows economic implementation and empirically still yields impressive performance. We will specify the values for $k$ and $N_{rep}$ for different experiments. 
\paragraph{Solvers for $\ell^1$ Regression} We perform high-dimensional Nearest Subspace (NS) search in $\ell^1$ (HDL1) as baseline. Considering the scale of $\ell^1$ regression in this case, we employ an Augmented Lagrange Method (ALM) numerical solver~\cite{YangA2010-pp} whenever the recognition performance is not noticeably affected (the case on extended Yale B below); otherwise we employ the more accurate interior point method (IPM) solvers~\cite{candes2005l1} (for the synthesized experiment and ALOI). All the instances of $\ell^1$ regression in the projected low dimensions are handled by interior point method (IPM) solvers. 

\subsection{Experiments with Synthesized Data} \label{sec:synthesized_exp}
We independently generated $n = 100$ random subspaces in $\R^{10000}$ (i.e., $D = 10000$), each of which is $5$-dimensional (i.e., $r = 5$). Each subspace is generated as the column span of an $D \times r$ iid standard normal matrix. We also prepared a pool of $k = 100$ Cauchy matrices of dimension $d \times D$, where $d$ takes values in $\left\{10, 30, 50, 70, 90\right\}$. 

To verify our theory (Theorem~\ref{thm:main}), we randomly picked one subspace, and generate a sample $\mb y = \mb B \mb x$, where $\mb B$ is one orthonormal basis for the subspace and $\mb x$ contains iid standard normal entries. To induce reasonable distance gap, and also simulate some sparse errors, we divided $\mb y$ by the magnitude of its largest entries, and added errors that is uniformly distributed in $\left[-1, 1\right]$ to a $\theta$-fraction of $\mb y$'s entries, i.e., we got $\hat{\mb y} = \mb y + \mb e_{\theta}$. We varied $\theta$ from $0.05$ with $0.3$, with $0.05$ as step size. Growth in fraction of corruption diminishes the distance gap $\eta$, as evidenced from the legend of the left subfigure in Figure~\ref{fig:syn_exp}. To estimate the success probability of low-dimensional regression to retrieve the \emph{nearest} (in principle not necessarily the originating) subspace, in each setting we exhausted our pool of projection matrices and obtained the empirical success rate. Left subfigure of Figure~\ref{fig:syn_exp} reports the results. Note that here $r\log n \approx 23$, when the distance gap is not so small, say $\eta > 2$, $d = 30$ actually enjoys at least $50\%$ chance to preserve the nearest subspace. Also reasonably to get the same level of success probability, small distance gaps evidently entails large projection dimensions. 

\begin{figure}[!htbp]
\centering
\includegraphics[width = 0.45\linewidth]{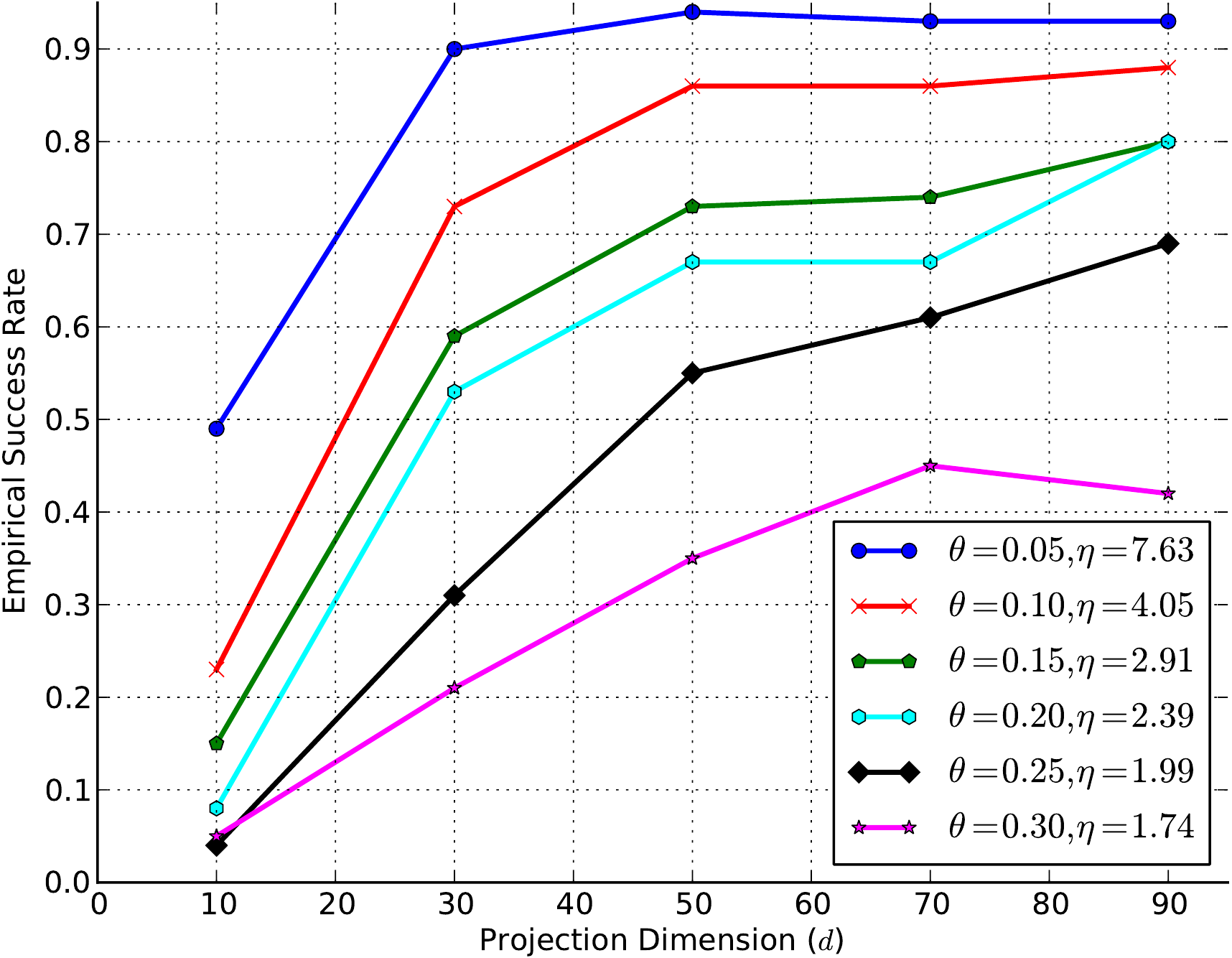}
\includegraphics[width = 0.45\linewidth]{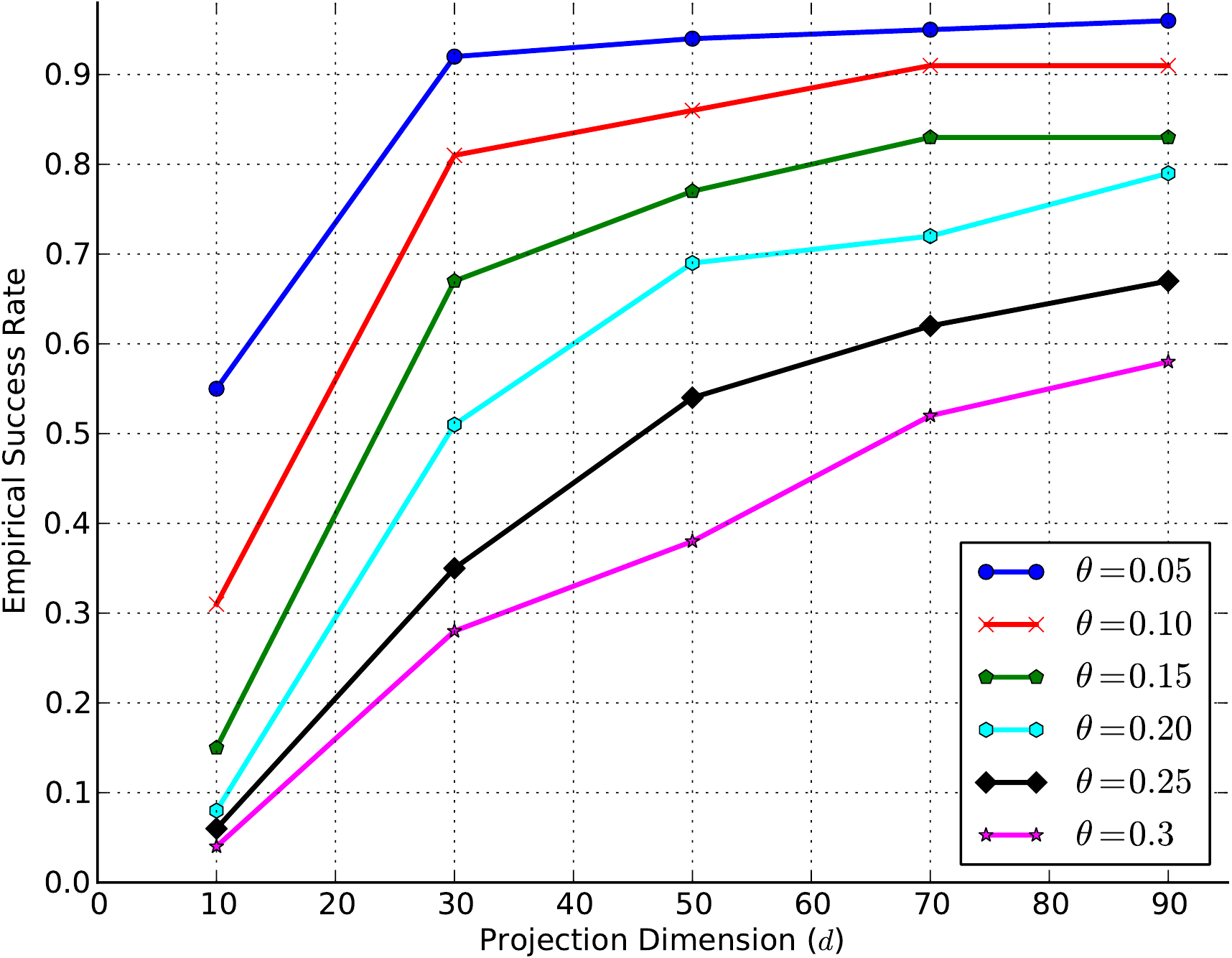}
\caption{Left: Probabilities of preserving the nearest subspaces by different projection dimensions for a fixed sample corrupted by different levels of additive errors; Right: Fraction of samples that still identify their nearest subspace after random projections of different dimensions.}
\label{fig:syn_exp}
\end{figure}

To emulate visual recognition scenarios such as we will do in the next experiments, we independently randomly generated $500$ query points similar to $\hat{\mb y}$ and also varied $\theta$ similarly as above to induce different distance gaps. To keep things simple, for each query we randomly picked up one projection from the pool and omitted repetitions refined scanning altogether. The success probability is now defined as the fraction of samples that successfully identify their respective nearest subspaces in randomly chosen low-dimensional space. The right subfigure in Figure~\ref{fig:syn_exp} gives such results. Again even on this much trimmed version of our algorithm, $d = 30$ helps half of the samples find their nearest subspace when the corruption level is below $0.20$!

\subsection{Robust Face Recognition on Extended Yale B} 
Under certain physical assumptions, images of one person taken with fixed pose and varying illumination can be well-approximated using a nine-dimensional linear subspace~\cite{Basri2003-PAMI}. Because physical phenomena such as occlusions and specularities, as well as physical properties such as nonconvexity~\cite{Zhang2013-ICCV} may cause violation of the low-dimensional linear model, we formulate the recognition problem as one of finding the closest subspace to $\mb q$ in $\ell^1$ norm~\cite{Wright2009-PAMI}\footnote{In other words, we formulate the problem as $\ell^1$ NS search. This is different from the idea of sparse representation in SRC~\cite{Wright2009-PAMI} for face recognition. Since our focus here is not to propose a new or optimal face recognition algorithm (although $\ell^1$ NS method happens to be new for the task), we prefer to save detailed discussions in this line for future work. Nevertheless, our preliminary results indeed suggest $\ell^1$ NS is as competitive as SRC for the popular extended Yale B face recognition benchmark we have used here. }. 

The Extended Yale B face dataset~\cite{Georghiades2001-PAMI} (EYB, cropped version) contains cropped, well-aligned frontal face images ($168 \times 192$) of $38$ subjects under 64 illuminations ($2,432$ images in total, the $18$ corrupted during acquisition not used here). For each subject, we randomly divided the images into two halves, leading to $1205$ training images and $1209$ test images. To better illustrate the behavior of our algorithm, we strategically divided the test set into two subsets: moderately illuminated ($909$, \textbf{Subset M}) and extremely illuminated ($300$, \textbf{Subset E}). The division is based on the light source direction (\emph{wrt}. the camera axis): images taken with either azimuth angle greater than $90^\circ$ or elevation angle greater than $60^\circ$ would be classified as extremely illuminated \footnote{Note that this division does not closely match in any way the four subset division coming with the database, as described in~\cite{Georghiades2001-PAMI}. }. Since all faces are supposed to known, hence the closed-world assumption holds true in this setting. 

\paragraph{Recognition with Original Images} Figure~\ref{fig:EYB_moderate_vary_d} presents the evolution of recognition rate on \textbf{Subset M} as the projection dimension ($d$) grows \emph{with only one repetition of the projection} ($N_{rep} = 1$). 
\begin{figure}[!htbp]
\centering
\includegraphics[width = 0.5\linewidth]{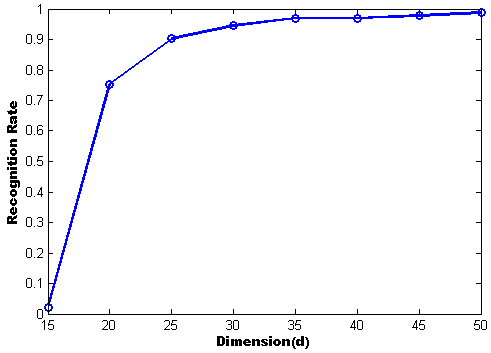}
\caption{Recognition rate versus projection dimension ($d$) \emph{with one repetition} on \textbf{Subset M} face images of EYB. The recognition rate stays stable above $90\%$ with $d \geq 25$. The high-dimensional NS in $\ell^1$ achieves perfect ($100\%$) recognition. Note the ambient dimension in this case is $D = 168 \times 192 = 32256$. }
\label{fig:EYB_moderate_vary_d}
\end{figure}
We took the subspace dimension to be nine ($r = 9$) as conventional. Our experiment shows the HDL1 achieves perfect recognition ($100\%$) on this subset, implying recognition in this subset corresponds perfectly to NS search in $\ell^1$. So Figure~\ref{fig:EYB_moderate_vary_d} actually represents the evolution of ``average'' success probability for \emph{one repetition} over the subset. Suppose the distance gap $\eta$ is significant such that $1/\alpha \to 1$ (recall $\alpha$ is near $1-1/\eta$ in our Theorem~\ref{thm:main}), our theorem suggests that one needs to set roughly $d = r \log n = 9 * \log 38 \approx 33$ to achieve a constant probability of success. Our result is consistent with this theoretical prediction and the probability is already stable above $0.9$ for $d \geq 25$. With $3$ repetitions and $d=25$, the overall recognition rate is $99.56\%$ ($4$ errors out of $909$), nearly perfect. Figure~\ref{fig:EYB_moderate_fail} presents the failing cases. 
\begin{figure}[!htbp] 
\centering
\includegraphics[width = 0.6\linewidth]{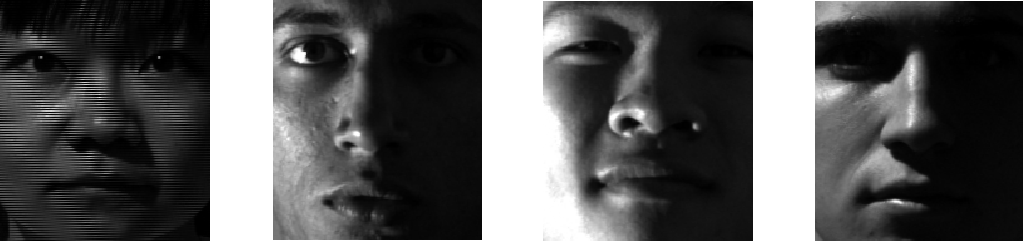}
\caption{Failing cases of our method on \textbf{Subset M} of EYB. }
\label{fig:EYB_moderate_fail}
\end{figure}
They either contain significant artifacts or approach the extremely illuminated cases, the failing mechanism and remedy of which are explained below. 

For extremely illuminated face images, the $\ell^1$ distance gap between the first and second nearest subspaces is much less significant (one example shown in Figure~\ref{fig:EYB_extreme_gap}). 
\begin{figure}[!htbp]
\centering
\includegraphics[width = 0.5\linewidth]{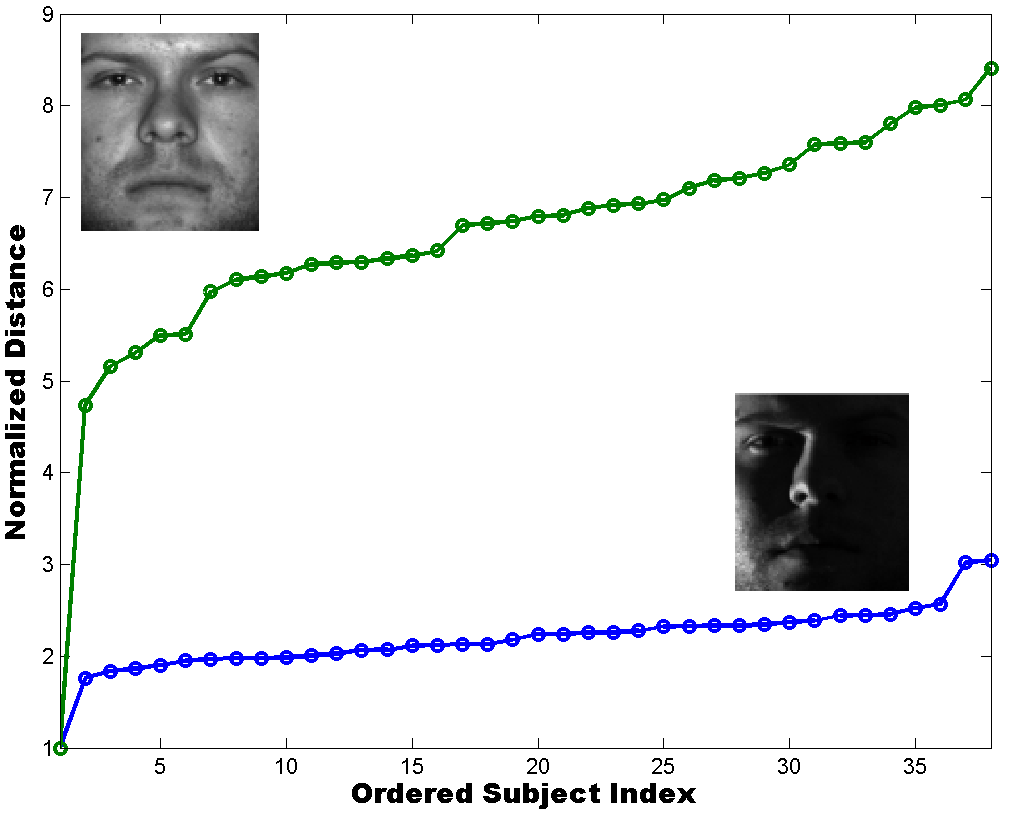}
\caption{Samples of moderately/extremely illuminated face images and their $\ell^1$ distances to other subject subspaces. The subjects have been ordered in ascending order of $\ell^1$ distance from the sample and the distances are normalized such that the first distance is $1$. Note that for the moderately illuminated sample, a distance gap of about $4.8$ is observed while this is only about $1.8$ for the extremely illuminated sample. }
\label{fig:EYB_extreme_gap}
\end{figure}
Our theory suggests $d$ should be increased to compensate for the weak gap (because the exponent $1/\alpha$ becomes significant). Our experimental results confirm this prediction. Specifically, for $r = 15$ (we took this to be higher than $9$ to account for the great variation due to extreme illuminations in this case), the HDL1 achieves $94.7\%$ accuracy while our method achieves only $79.3\%$ when $d = 25$ and $N_{back} = 5$ ($N_{back}$ is the number of back-research, i.e., ``refined scanning'' as in the algorithm description, in high dimensions). The recognition rate is boosted significantly when we increase $d$, or increase $N_{back}$ (this is another way of amplifying the success probability),   
\begin{table}[!htbp]
\caption{Recognition Rate ($\%$) on \textbf{Subset E} of EYB with varying $d$ and $N_{back}$. }
\label{table:EYB_extreme_rate}
\centering
\begin{tabular}{c||l|ccc}
\hline
	  			&HDL1	&$d=25$ 	&$d=50$		&$d=70$\\\hline
$r=15,N_{back}=5 $	&94.7	&79.3 &87.7	& 92.3 \\\hline
$r=15,N_{back}=10$	&94.7	&87.3	&92.0	&94.0 \\\hline
\end{tabular}
\end{table}
as evident from Table~\ref{table:EYB_extreme_rate}.

\paragraph{Recognition on Artificially Corrupted Images} 
In order to illustrate the robustness of $\ell^1$ NS approach for recognition and particularly the capability of our method to preserve such property of $\ell^1$, we emulated the robust recognition experiment on artificially corrupted images, as done in~\cite{Wright2009-PAMI}. To be specific, Subset $1$ and Subset $2$, which comprise images taken under near-frontal illuminations, are used for training; and Subset $3$ is used for testing.\footnote{The subset division completely matches the division in~\cite{Georghiades2001-PAMI}, which can also be found online: \url{http://cvc.yale.edu/projects/yalefacesB/subsets.html}. } We corrupted each original test image with (1) randomly-distributed sparse corruptions, and (2) structured occlusions. For the first setting, we replaced, respectively, $10\%$ to $90\%$ (with $10\%$ resolution) of randomly chosen pixels of the test images with i.i.d. uniform integer values in $\left[0, 255\right]$\footnote{In other words, any valid pixel value for 8-bit gray-scaled image.}. For the second, the \emph{mandril} image is scaled to, again $10\%$ to $90\%$ (with $10\%$ resolution), of the image size, and imposed on the image with randomly chosen locations. 
\begin{figure}[!htbp]
\centering
\includegraphics[width = 0.2\linewidth]{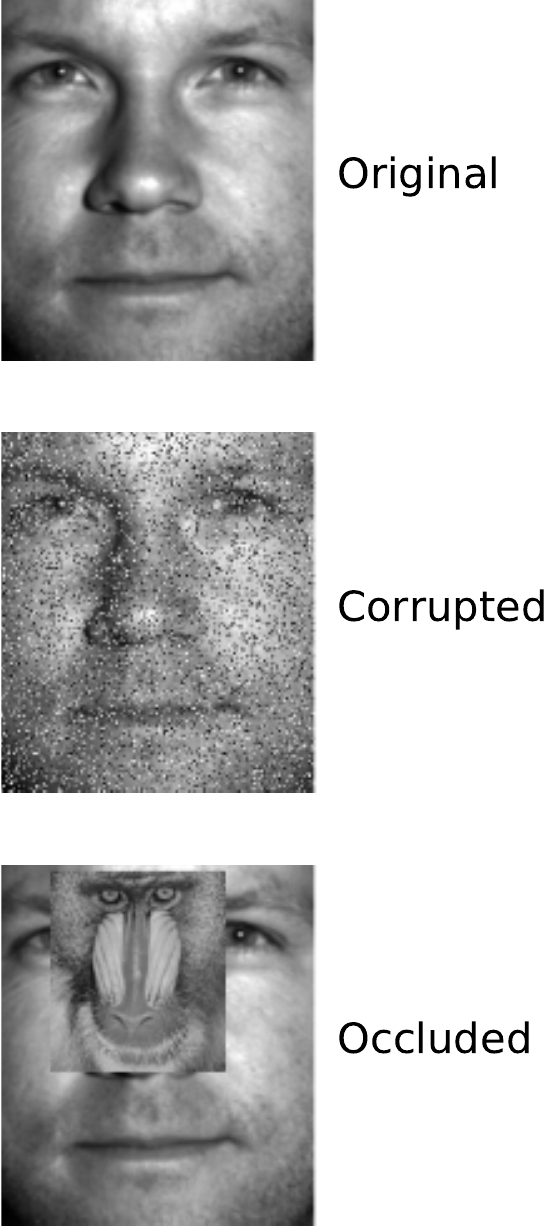}
\includegraphics[width = 0.6\linewidth]{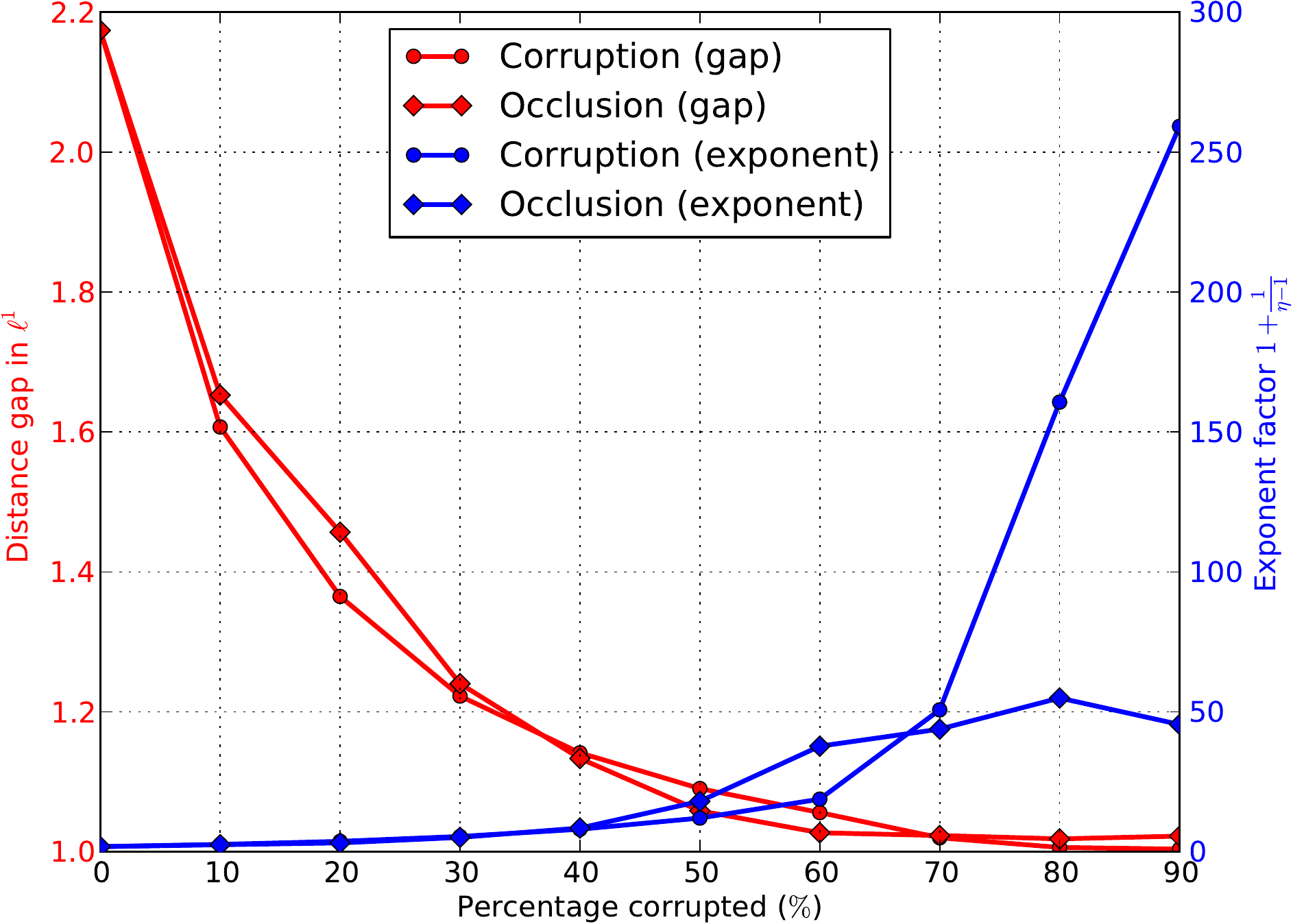}
\caption{Left: Sample of original images and the corrupted versions. In both corrupted images $20\%$ of the pixels are contaminated. Right: Evolution of the distance gap due to corruptions and the corresponding exponents calculated as $1+\frac{1}{\eta - 1}$ (in accordance with theorem~\ref{thm:main}). The distance gap is estimated by taking one random example from each test subject. }
\label{fig:EYB_corruption_example}
\end{figure}
Figure~\ref{fig:EYB_corruption_example} shows some typical samples of both cases, and also the effect of corruptions on distance gaps - corruptions significantly weaken the gaps. In particular, the gap drops to $1$ very rapidly as the corruption level increases, suggesting according to our theory that significant dimension reduction via projection is not likely beyond low corruption levels (say $20\%$ from the plot). 

To get a flavor of the level of approximation, we fix $k = 100$, $N_{rep} = 5$, $r = 9$, $N_{back} = 5$ and compare the HDL1 with our approximation scheme (dubbed LDL1) for $d = 100$, $d = 200$, and $d = 300$, respectively. To demonstrate the advantage of $\ell^1$ norm in terms of stability against corruptions, we also include comparison with the very natural $\ell^2$ NS variant (dubbed HDL2)\footnote{This is exactly the nearest subspace classifier that was compared to the SRC classifier in~\cite{Wright2009-PAMI}.}. 
\begin{figure}[!htbp]
\centering
\includegraphics[width = 0.45\linewidth]{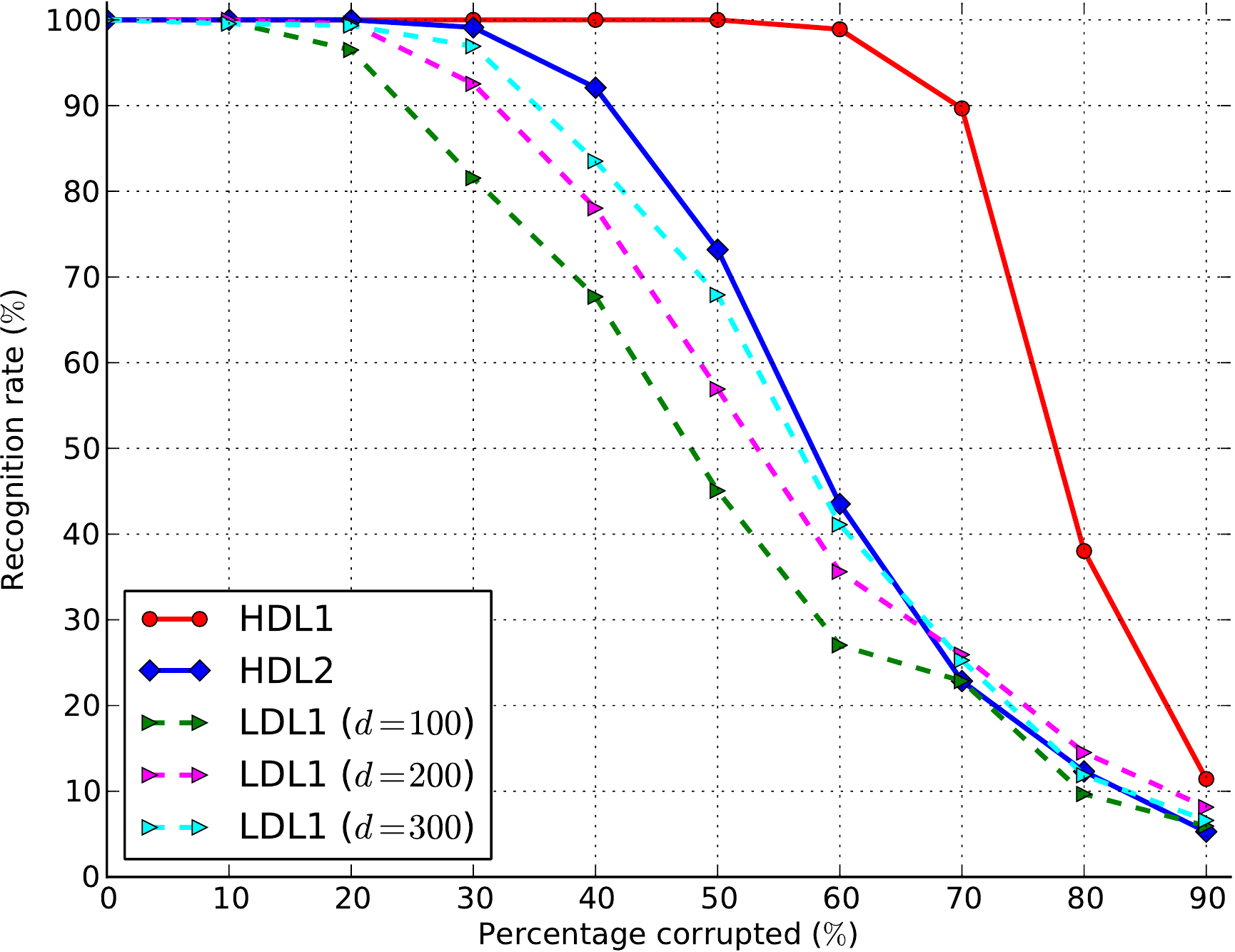}
\includegraphics[width = 0.45\linewidth]{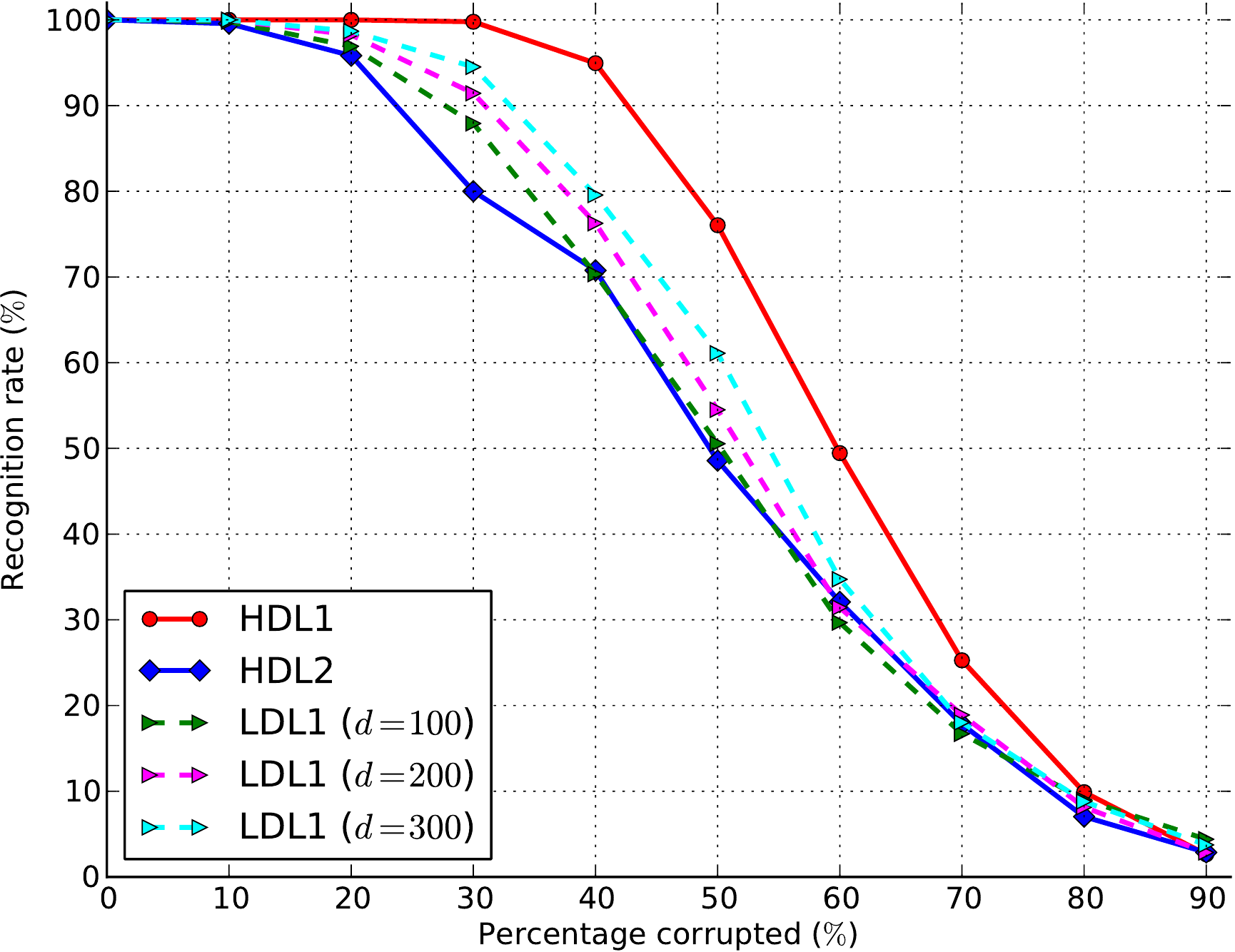}
\caption{Recognition rate under corruptions for on EYB. ($k = 100$, $N_{rep} = 3$, $r = 9$, throughout the experiments.) Left: under random corruptions; Right: under structured occlusion.}
\label{fig:EYB_corruption_result}
\end{figure}
Figure~\ref{fig:EYB_corruption_result} summarizes the recognition performances for each setting. Our method exhibits comparable level of performance with the HDL1 for corruptions less than or equal to $20\%$ and observable performance lag beyond that level. This is a reasonable price to pay as we insist on working in low dimensions for efficiency. In our current setting of the dimension, the performance of LDL1 (not HDL1) is even worse than HDL2 for the random corruption model, in particular when the corruption level is high. For the structured occlusion model, LDL1 is consistently better than HDL2. Increasing $d$ is likely to improve the approximation accuracy further.    

\subsection{Object Instance Recognition} To investigate the applicability of our proposal for large-scale recognition tasks, we took a subset of the multi-purpose Amsterdam Library of Object Images (ALOI) library~\cite{geusebroek2005amsterdam}\footnote{Available online: \url{http://staff.science.uva.nl/~aloi/}. }. This subset comprises images of $1000$ toy-like objects with fixed pose, taken under $24$ different illumination directions for each object, and hence includes $24$ images per object. We randomly took $12$ images of each object for training, and the rest for test. Although these objects in general have nonconvex shapes and non-Lembertian reflectance property, we still approximate the collection of images of each object with a nine-dimensional subspace as proposed in~\cite{Basri2003-PAMI}. This again turns the recognition problem naturally into a subspace search problem. 

Again we are interested in robust recognition. We added random corruption of varying percentage ($10\% \sim 70\%$) to the test images, similar to the above for face images. We fixed $r = 9$, $k = 100$, $N_{rep} = 30$, $N_{back} = $. Table~\ref{table:ALOI_corruption_result} compares the performance of HDL1 and HDL2 under image corruption.  
\begin{table}[!htbp]
\caption{Recognition rate under corruptions for the selected (fixed pose but varying illumination conditions) ALOI subset ($r = 9$, $k = 100$, $N_{rep} = 30$). }
\label{table:ALOI_corruption_result}
    \begin{tabular}{c|cccccccc}
    \hline 
    Corruption Level ($\%$)       & 0     & 10     & 20     & 30    & 40     & 50     & 60    & 70   \\ \hline \hline 
    HDL1 ($\%$)                   & 99.35 & 99.40  & 99.42  & 99.45 & 99.47  & 99.24  & 43.33 & 1.85 \\
    HDL2($\%$)                    & 99.72 & 96.29  & 59.22  & 24.30 & 7.87   & 1.68   & 0.53  & 0.13 \\
    LDL1($\%$, $d = 200$)  & 99.41 & 99.10  & 89.54  & 66.74 & 42.62     & ---      & ---     & ---    \\ \hline \hline 
    Distance Gap ($\tilde{\eta}$) & 4.2858    & 1.3912 & 1.2074 & 1.1339    & 1.0833 & 1.0476 & 1.0117     & ---    \\
    \hline 
    \end{tabular}

\end{table}

The $\ell^1$ NS method again exhibits impressive tolerance to these corruption, as compared to the $\ell^2$ variant.\footnote{Systematic report of recognition results on ALOI is rare, with many only on a subset, say $300$ objects, perhaps because of the significant scale. One exception is~\cite{elazary2010bayesian}, which reports recognition performance under many different settings with state-of-the-art visual recognition schemes. Particularly relevant to our result here is they evaluated recognition on the illumination subset we choose here with the biologically-inspired HMAX model. With $25\%$ of the data for training, they achieved $83.13\%$ recognition rate. } In particular, HDL1 tolerates corruptions up to $50\%$ almost perfectly, on the test set. By comparison, HDL2 fails badly for corruption level beyond $10\%$. Our approximation scheme, LDL1 with $d = 200$, turns out to be effective for corruptions lower than $20\%$ (remains almost $ \geq 90\%$ correct), and fails gradually beyond that. We did not try higher projection dimensions, as 1) the computational burden would expand rapidly, and 2) from the estimate in Figure~\ref{fig:EYB_corruption_example}, the exponent associated with the predicted dimensions by our theory would be significant for distance gap lower than $1.2$, leading to significant demand for large $d$. 

\subsection{Some Results on Running Time} \label{sec:running_time}

It is obvious the running time of our algorithm is largely determined by how fast we can solve the $\ell^1$ regression problem, i.e., $\min \norm{\mb y - \mb A\mb x}{1}$ for $\mb A \in \R^{\tilde{d} \times r}$, the cost of which will be denoted as $T_{\ell^1}\left(\tilde{d}, r\right)$. To be concrete, in our recognition tasks for object instance recognition, the straightforward exhaustive search in the high dimension $\R^D$ costs a total of $nT_{\ell^1}\left(D, r\right)$, whereas the two level search algorithm we propose costs $n N_{rep}T_{\ell^1}\left(d, r\right) + N_{back} T_{\ell^1}\left(D, r\right)$ if we project onto a lower-dimensional $\R^d$ and repeat $N_{rep}$ to boost the success probability, and then select the best $N_{back}$ for the refined scanning in the original space. So the proposed algorithm will be practically interesting when $T_{\ell^1}\left(d, r\right) \ll T_{\ell^1}\left(D, r\right)$. 

We first experimented with simulated examples. We generate $\mb A$ as an orthonormal basis for an $r$-dimensional subspace in $\R^D$, where $D = 2^\rho$ and $\rho$ varies from $4.5$ to $15$ with $0.5$ step size, $r = 10$. For each $D$, $\mb x_0 \in \R^r$ is generated as iid Gaussians, and $y_0 = \mb A \mb x_0$. We then perform normalization and corruption addition, the same as we did in Section~\ref{sec:synthesized_exp}, with the fraction of corruption $\theta$ taken from $\left\{0.2, 0.4, 0.6, 0.8, 1.0\right\}$. 
\begin{figure}[!htbp]
\centering
\includegraphics[width = 0.6\linewidth]{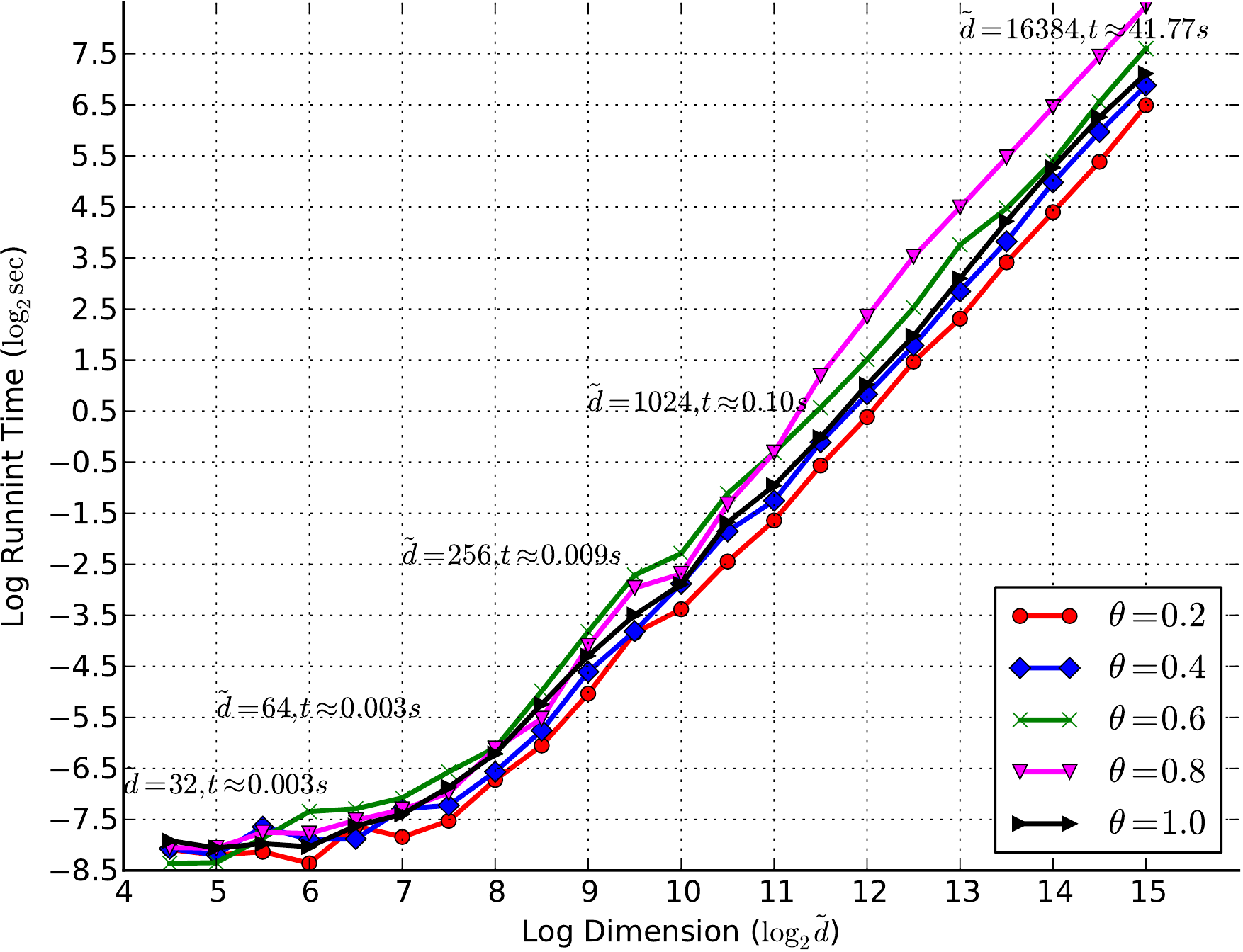}
\caption{(Log) Running Time vs. (Log) Dimension for the Simulated $\ell^1$ Regression Problems. It seems for the customized IPM solver, the complexity scales like $O\left(\tilde{d}^2\right)$. The computer runs 64bit Ubuntu 13.10, with Linux Kernel 3.11.0-17 and 64 bit Matlab 2012b. The processor is Xeon E5607 2.27G, and the RAM is 12G. We did the simulation using only one thread by turning on the {\tt -singleCompThread} flag for Matlab. }
\label{fig:l1_reg_time}
\end{figure}
We take the $\ell^1$ regression solver from $\ell^1$ magic~\cite{candes2005l1}, which implements the customized IPM outlined in Section 11.8.2 of~\cite{Boyd}. Figure~\ref{fig:l1_reg_time} plots the running time (in sec) vs. dimension ($\tilde{d}$), both in based-2 logarithm. To make the comparison fair as possible, we have turned on the {\tt -singleCompThread} flag to ensure Matlab is only using one thread for the simulation. It seems the running time scales approximately as $O\left(\tilde{d}^2\right)$. To see how that is relevant to our recognition problem, for $\theta = 0.2$, $T_{\ell^1}\left(256, 10\right) = 0.009 s$, whereas $T_{\ell^1}\left(16384, 10\right) = 41.77 s$. The running time differs by several orders of magnitude, leaving our algorithm significant advantage! 

To illustrate what this means in practice, we take a random instance from the Yale B recognition task with $10\%$ random corruptions and take $d = 100$. Previous experiment has confirmed this projection dimension works well for this case (see Figure~\ref{fig:EYB_corruption_result}). Again we take $N_{rep} = 5$ and $N_{back} = 5$, for the single-thread simulation, the high dimension exhaustive search costs $7496$ sec's, while the our two-level search algorithm only needs $467$ sec's \footnote{These daunting numbers can be significantly cut down by exploiting multi-core/GPU programming. We have exploited multicore programming in our actual experiments over the recognition tasks. }, over $16$ times faster! The cost of our algorithm is largely dictated by $N_{back}$ (empirically even smaller than because of potential ties). In larger dataset, when $N_{back}$ can be taken to be much smaller relative to $n$, the advantage could be more significant. 

\appendix

\section{Notation and Preliminaries} \label{app:preliminary}
We present detailed proofs to our technical lemmas throughout the appendix section. This part will provide some essential facts about stable distributions, in particular the Cauchy distribution. RV is short for random variable. 

\begin{definition}[Stable Distributions, page 43 of~\cite{uchaikin1999chance}]
An RV $Y$ is stable if and only if for arbitrary constants $c_1$ and $c_2$ there exist constants $a$ and $b$ such that 
\begin{align}
c_1 Y_1 + c_2 Y_2 \equiv_d a + b Y, 
\end{align}
where $Y_1 \equiv_d Y_2 \equiv_d Y$. It is said to be strictly stable if and only if $c_1 Y_1 + c_2 Y_2 \equiv_d b Y$ (i.e., one can take $a = 0$.). 
\end{definition}

\begin{theorem}[Characteristic Function of Stable Distributions, Theorem C.2 of~\cite{zolotarev1986one}]
A nondegenerate distribution $G$ is stable if and only if its characteristic function $\psi_G\left(t\right)$ satisfies: 
\begin{align}
\log \psi_G\left(t; \alpha, \beta, \gamma, \lambda\right) = \lambda\left(it\gamma - \left|t\right|^{\alpha} + i t\omega_A\left(t, \alpha, \beta\right)\right), 
\end{align}
where the real parameters $\alpha \in \left(0, 2\right]$, $\beta \in \left[-1, 1\right]$, $\gamma \in \left(-\infty, \infty\right)$ and $\lambda \in \left(0, \infty\right)$ and 
\begin{align}
\omega_A\left(t, \alpha, \beta\right)
= \begin{cases}
\left|t\right|^{\alpha - 1} \beta \tan\left(\pi \alpha/2\right) & \text{if $\alpha \neq 1$} \\
-\beta \left(2/\pi\right) \log\left|t\right| & \text{if $\alpha  = 1$}. 
\end{cases}
\end{align}
\end{theorem} 

We will use $G^A\left(x; \alpha, \beta, \gamma, \lambda\right)$ to denote the stable distribution with characteristic function $\psi_A\left(t; \alpha, \beta, \gamma, \lambda\right)$, following the convention in~\cite{uchaikin1999chance}. Also we write $G^A\left(x; \alpha, \beta\right)$ when $\gamma = 0$ and $\lambda = 1$, thinking of this setting as the canonical form. 

\begin{definition}[(Symmetric) $\ell^p$-Stable Distributions] \label{def:ell_p_stable}
An RV $X$ is called symmetric $\ell^p$-stable for some $p \in \left(0, 2\right]$ if the characteristic function 
\begin{align}
\psi_X\left(t\right) = \exp\left(-c\left|t\right|^p\right)
\end{align}
for some $c >0 $ and for all $t \in \R$. Its distribution is called symmetric $\ell^p$-stable distribution. 
\end{definition}

By comparing the characteristic functions, it is clear a symmetric $\ell^p$-stable distribution is the stable distribution $G^A\left(x; p, 0, 0, c\right)$ for some $c > 0$. It is also obvious that $\ell^p$ stable distributions exist for all $p \in \left(0, 2\right]$ by virtue of the existence of the stable distribution with the corresponding parameters. 

\begin{lemma}[Property of (Symmetric) $\ell^p$-Stable Distributions]
Consider iid RV's $X_1, \cdots, X_n$ obeying a symmetric $\ell^p$-stable distribution. Then for any real sequence $\left\{c_i\right\}_{i \in [n]}$, we have 
\begin{align}
\sum_{i=1}^n c_i X_i \equiv_d \left(\sum_{i=1}^n \left|c_i\right|^p\right)^{1/p} X, 
\end{align}
where $X$ has the same distribution as $X_i$'s. 
\end{lemma}
\begin{proof}
Assume the characteristic function of $X_i$'s are $\psi\left(t\right) = \exp\left(-c\left|t\right|^p\right)$ for some $c > 0$. Then 
\begin{align}
& \psi_{\sum_{i=1}^n c_i X_i}\left(t\right) 
=  \expect{\exp\left(it\sum_{i=1}^n c_i X_i\right)}
=  \prod_{i=1}^n \expect{\left(itc_iX_i\right)} \\
= \; & \prod_{i=1}^n \exp\left(-c\left|c_i\right|^p \left|t\right|^p \right) = \exp\left(-c \sum_{i=1}^n \left|c_i\right|^p \left|t\right|^p\right) \\
= \; &  \expect{\exp\left(it\left(\sum_{i=1}^n \left|c_i\right|^p\right)^{1/p} X\right)} = \psi_{\left(\sum_{i=1}^n \left|c_i\right|^p\right)^{1/p} X}\left(t\right), 
\end{align}
completing the proof.  
\end{proof}

We will henceforth omit the word ``symmetric'' for simplicity when considering $\ell^p$-stable distributions. In fact, we will deal exclusively with the standard Cauchy RV's $X \sim \mc C\left(0, 1\right)$ with PDF $p_{\mc C}\left(x\right) = \frac{1}{\pi} \frac{1}{1+x^2}$ and the standard half-Cauchy RV's $X \sim \mc{HC}\left(0, 1\right)$ with PDF 
\begin{equation}
p_{\mc {HC}}(x) = 
\begin{cases}
\frac{2}{\pi} \frac{1}{1+x^2} & x \geq 0 \\
0  &  x < 0 
\end{cases}.  
\end{equation}
One remarkable aspect of the standard Cauchy is it is $\ell^1$-stable. Furthermore, by inverting the characteristic function as stated in Definition~\ref{def:ell_p_stable}, one can see all $\ell^1$-stable distribution has to be standard Cauchy or its scaled version (controlled by $c$) \cite{uchaikin1999chance}. These facts are fundamental to our subsequent analysis. In addition, the following two-sided bound for upper tail of a half-Cauchy RV will also be useful.
\begin{lemma}
For $X \sim \mc{HC}\left(0, 1\right)$, we have $\forall t \geq 1$
\begin{equation}
 \frac{1}{\pi} \frac{1}{t} \leq \prob{X\geq t} \leq \frac{2}{\pi} \frac{1}{t}. 
\end{equation}
\end{lemma}

\begin{proof}
We have 
\begin{align}
 \frac{1}{\pi} \frac{1}{t} =\frac{2}{\pi}\integral{t}{\infty}{\frac{1}{2x^2}}{dx} & \leq  \prob{X \geq t} = \frac{2}{\pi}\integral{t}{\infty}{\frac{1}{1+x^2}}{dx} \\  &\leq \frac{2}{\pi}\integral{t}{\infty}{\frac{1}{x^2}}{dx} = \frac{2}{\pi}\frac{1}{t}. 
\end{align} 
In fact, the upper bound holds for any $t > 0$. \qquad 
\end{proof}
\newline
For any matrix $\mb A$, we will use $\mb A_{i*}$ to denote its $i^{th}$ row, and $\mb A_{*j}$ its $j^{th}$ column. 
\section{Proof of Lemma~\ref{lemma:good_subspace}} \label{app:proof_expansion}
We first describe the behavior of sum of iid half-Cauchy's in the limit, based on the \emph{generalized central limit theorem} (GCLT), which we record below for the sake of completeness. 

\begin{theorem}[GCLT, Page 62 in~\cite{uchaikin1999chance}]
Let $X_1, \cdots, X_n$ be iid RV's with the distribution function $F_X\left(x\right)$ satisfying the conditions 
\begin{align}
1 - F_X\left(x\right) \sim cx^{-\mu}, & \quad x \to \infty \\
F_X\left(x\right) \sim d\left|x\right|^{-\mu},  & \quad x \to -\infty, 
\end{align}
with $\mu > 0$\footnote{Note that there are obvious typographical errors in (2.5.17) and (2.5.18) in the original theorem statement. This can be seen from, e.g., Theorem 2 of $\S 35$ of Chapter $7$ in~\cite{gnedenko1968limit}. }. Then there exist sequences $a_n$ and $b_n > 0$ such that the distribution of the centered and normalized sum 
\begin{equation}
Z_n = \frac{1}{b_n}\left(\sum_{i=1}^n X_i - a_n\right)
\end{equation}
weakly converges to the stable distribution with parameters 
\begin{equation}
\alpha = \begin{cases}
\mu & \mu \leq 2 \\
2 & \mu > 2
\end{cases}, 
\quad 
\beta = \frac{c-d}{c+d} 
\end{equation}
as $n \to \infty$: $F_{Z_n}\left(x\right) \Rightarrow G^A\left(x; \alpha, \beta\right)$. In particular, when $\mu = 1$, one can take 
\begin{equation}
a_n = \beta\left(c + d\right) n \log n, \quad b_n = \frac{\pi}{2}\left(c+d\right) n. 
\end{equation}
\end{theorem}
\begin{lemma}
Let $X_1, \cdots, X_n$ be iid half-Cauchy RV's. Consider the sequence 
\[
Z_n = \left(\sum_{i=1}^n X_i - \frac{2}{\pi}n \log n\right)/n. 
\]
One has 
\begin{equation}
F_{Z_n}\left(x\right) \Rightarrow G^A\left(x; 1, 1\right). 
\end{equation}
\end{lemma}
\begin{proof}
We proceed by determining the parameters $\mu, c, d, \alpha, \beta$ and sequences $a_n$ and $b_n$ as appearing in the GCLT above. For any half-Cauchy RV $X$, we have 
\begin{equation}
1 - F_X\left(x\right) = \frac{2}{\pi}\integral{x}{\infty}{\frac{1}{1+x^2}}{dx} = \frac{2}{\pi} \left(\frac{\pi}{2} - \arctan x\right) = \frac{2}{\pi} \arctan \frac{1}{x}. 
\end{equation}
When $x \to \infty$, $\left|\frac{1}{x}\right| \leq 1$. We expand $\arctan \frac{1}{x}$ into an infinite series 
\begin{equation}
1 - F_X\left(x\right) = \frac{2}{\pi} \arctan \frac{1}{x} = \frac{2}{\pi} \sum_{m=0}^{\infty} \frac{\left(-1\right)^m \left(1/x\right)^{2m+1}}{2m+1} \sim \frac{2}{\pi} \frac{1}{x} \; \text{as} \; x \to \infty. 
\end{equation}
So we have $\mu = 1$, $c = \frac{2}{\pi}$. Since $F_X\left(x\right) = 0$ for any $x \leq 0$, $d = 0$. Hence we have 
\begin{equation}
\alpha = \mu = 1, \; \beta = \frac{c+d}{c-d} = 1, 
\end{equation}
with the centering and normalizing sequences 
\begin{equation}
a_n = \beta \left(c+d\right) n \log n = \frac{2}{\pi} n \log n, \; b_n = \frac{\pi}{2}\left(c+d\right) n = n. 
\end{equation}
Hence the sequence $Z_n$ converges weakly to $G^A\left(x; 1, 1\right)$ in distribution. 
\end{proof}

A plot\footnote{We use implementation available online \url{http://math.bu.edu/people/mveillet/html/alphastablepub.html}. The convention used here (designated with subscript ``ST'') is almost identical to Zolotarev's form A (designated with subscript ``A") in~\cite{uchaikin1999chance}, with the following correspondences: $\alpha_{ST} = \alpha_A$, $\beta_{ST} = \beta_A$, $\gamma_{ST} = \lambda_A$, $\delta_{ST} = \gamma_A\lambda_A$. } of $G^A\left(x; 1, 1\right)$ is included in Figure~\ref{fig:stable_g11}, which will be useful to the following proof. 

\begin{figure}[!htbp]
\centering
\includegraphics[width = 0.5\linewidth]{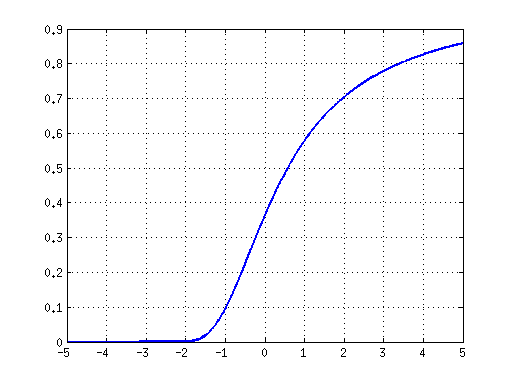}
\caption{Plot of the stable distribution $G^A\left(x; 1, 1\right)$. }
\label{fig:stable_g11}
\end{figure}
\begin{proof}[of Lemma~\ref{lemma:good_subspace}]
By $\ell^1$ stability of Cauchy, we have
\begin{equation}
\norm{\mb P\mb w}{1} = \norm{\sum_{i=1}^d \mb P_{i*}\mb w}{1} \equiv_d \norm{\mb w}{1} \norm{\sum_{i=1}^d \Psi_i}{1} \equiv_d \norm{\mb w}{1} \sum_{i=1}^d \Phi_i, 
\end{equation}
where $\Psi_1, \cdots, \Psi_d$ are iid Cauchy and $\Phi_1, \cdots, \Phi_d$ their corresponding half-Cauchy's. So we are interested in behavior of the sequence 
\begin{equation}
p_d \doteq \prob{\sum_{i=1}^d \Phi_i \leq \frac{2}{\pi} d \log d} = \prob{\frac{\sum_{i=1}^d \Phi_i - \frac{2}{\pi} d \log d}{d} \leq 0}. 
\end{equation}
Again we consider the sequence $S_d = \left(\sum_{i=1}^d \Phi_i - \frac{2}{\pi} d \log d\right)/d$. Since $S_d \Rightarrow G^A\left(x; 1, 1\right)$ as $d \to \infty$, and all stable distributions have continuous distribution function, we have for $x = 0$, 
\begin{equation}
p_d = \prob{S_d \leq 0} \to G^A\left(0; 1, 1\right) \; \text{as} \; d \to \infty. 
\end{equation}
So there exists $N \in \N$, such that $\forall d > N$, $p_d \geq 0.3$, where we observe that the numerical value $G^A\left(0; 1, 1\right)$ is strictly greater than $0.3$. So one can take the numerical constant $c$ in the lemma as 
\begin{equation}
c = \min\left(p_2, \cdots, p_N, 0.3\right) \geq 0. 
\end{equation} 
To see $c > 0$, note that $\forall d \in \N \setminus \left\{1\right\}$, 
\begin{align}
p_d 
& = \prob{\sum_{i=1}^d \Phi_i \leq \frac{2}{\pi} d \log d} \geq \prob{\Phi_i \leq \frac{2}{\pi} \log d, \forall i \in [d]} \\
& = \left(\prob{\Phi_1 \leq \frac{2}{\pi} \log d}\right)^d = \left[\frac{2}{\pi} \arctan\left(\frac{2}{\pi} \log d\right)\right]^d > 0. 
\end{align}
Hence we complete the proof. \qquad 
\end{proof}
\section{Proof of Lemma~\ref{lemma:contract}} \label{app:proof_contract}
We will use $\mb 1_{\mathrm{conditional}}$ as indicator function that assumes either $1$ (when the conditional is asserted) or $0$ (otherwise). 

\begin{proof}[{\bf of Lemma~\ref{lemma:contract}}]
Similar to the above it is enough to bound $\sum_{i=1}^d \Phi_i$. For the integer grid $1 < 2 < \cdots < k$, we have 
\begin{equation}
 \Phi_i \geq  \mb{1}_{\Phi_i \geq 1} + \mb{1}_{\Phi_i \geq 2}  + \cdots +  \mb{1}_{\Phi_i \geq k}
\end{equation}
and hence 
\begin{equation} \label{eq:lemma32_lower_grid}
 \sum_{i=1}^{d} \Phi_i \geq \sum_{j=1}^k \sum_{i=1}^{d} \mb{1}_{\Phi_i \geq j}. 
\end{equation}
Notice that $\vartheta_j \doteq \sum_{i=1}^{d} \mb{1}_{\Phi_i \geq j}$ is the sum of $d$ independent Bernoulli RV's with rate  $\prob{\Phi_1 \geq j}$ and hence $\expect{\vartheta_j} = d\prob{\Phi_1 \geq j}$. An application of the Chernoff bound gives us 
\begin{equation}
\prob{\vartheta_j < \left(1-\delta\right)d\prob{\Phi_1 \geq j}} \leq \exp\left(-\frac{\delta^2d\prob{\Phi_1 \geq j}}{2}\right). 
\end{equation} 
Now suppose the event that $\vartheta_j \geq \left(1-\delta\right)d\prob{\Phi_1 \geq j}$ for all $j \in [k]$ occurs, we would have 
\begin{align}
\quad \sum_{i=1}^{d} \Phi_i 
& \geq \sum_{j=1}^k \sum_{i=1}^{d} \mb{1}_{\Phi_i \geq j} = \sum_{j=1}^{k} \vartheta_j  \quad \text{(by~\eqref{eq:lemma32_lower_grid} and definition of $\vartheta_k$)} \\
& \geq d\left(1-\delta\right)\sum_{j=1}^k \prob{\Phi_1 \geq j} \quad \text{(by our assumption above)}\\
& = d\left(1-\delta\right) \frac{2}{\pi}\sum_{j=1}^k \integral{j}{\infty}{\frac{1}{1+x^2}}{dx} \quad \text{($\Phi_1$ is half-Cauchy)}\\
& = d\left(1-\delta\right) \frac{2}{\pi}\sum_{j=1}^k \arctan\left(1/j\right) \\
& \geq d\left(1-\delta\right) \frac{2}{\pi} \log(k + 1). \quad \text{(by Lemma~\ref{lem:sum_arctangent} below)} 
\end{align}
Hence 
\begin{align}
\prob{\sum_{i=1}^d \Phi_i  < \left(1-\delta\right)d \frac{2}{\pi}\log\left(k + 1\right)} 
& \leq \prob{\exists\; j \in [k], \vartheta_j < \left(1-\delta\right)d\prob{\Phi_1 \geq j}} \\
& \leq \sum_{j=1}^k \exp\left(-\frac{\delta^2d \prob{\Phi_1 \geq j}}{2}\right) \quad \text{(union bound)}\\
& \leq k\exp\left(-\frac{\delta^2 d}{2\pi k}\right). 
\end{align}
It is always true that 
\begin{align}
\prob{\sum_{i=1}^d \Phi_i < \left(1-\delta\right)d \frac{2}{\pi}\log d^{1-\alpha}} 
\leq & \;  \prob{\sum_{i=1}^d \Phi_i \leq \left(1-\delta\right)d \frac{2}{\pi}\log \left(\lfloor d^{1-\alpha} \rfloor + 1\right)}. 
\end{align}
Now by setting $k = \lfloor d^{1-\alpha} \rfloor \geq 1$ for the above bound we derived, we have 
\begin{align}
& \;  \prob{\sum_{i=1}^d \Phi_i \leq \left(1-\delta\right)d \frac{2}{\pi}\log \left(\lfloor d^{1-\alpha} \rfloor + 1\right)} \\
\leq & \; \lfloor d^{1-\alpha} \rfloor \exp\left(-\frac{\delta^2 d}{2\pi} \frac{1}{\lfloor d^{1-\alpha} \rfloor}\right) \\
\leq & \; d^{1-\alpha} \exp\left(-\frac{\delta^2 d}{2\pi} \frac{1}{ d^{1-\alpha}}\right), 
\end{align}
which leads to the result we have claimed. 
\qquad 
\end{proof}
\begin{lemma} \label{lem:sum_arctangent}
$\forall k \in \N$, $\sum_{j=1}^k \arctan\left(1/j\right) \geq \log\left(k+1\right)$. 
\end{lemma}
\begin{proof}
It is true for $k=1$ as $\pi/4 > \log(2)$. Now suppose the claim holds for $k-1$, i.e., $\sum_{j=1}^{k-1} \arctan\left(1/j\right) \geq \log\left(k\right)$, we need to show it holds for $k$. It suffices to show $\arctan(1/k) \geq \log\left(1+1/k\right)$. This follows from the fact that $\arctan x \geq \log \left(1 + x\right)$ for $x \in \left[0, 1\right]$. \qquad 
\end{proof}

We next show in some sense the bound we obtained in Lemma~\ref{lemma:contract} above cannot be significantly improved. 
\begin{lemma}
For any $d \in \N$ and any $\beta \in \left(0, 1\right)$ such that $d^{\beta} \geq 2$, if $\Phi_1, \cdots, \Phi_d$ are iid Half-Cauchy, then 
\begin{equation}
\prob{\sum_{i=1}^d \Phi_i \leq \frac{2}{\pi} \beta d \log d + O\left(d\right)} \geq \frac{\exp\left(-Cd^{1-\beta}\right)}{1 + \log d}, 
\end{equation}
where $C$ is some numerical constant. 
\end{lemma}
\begin{proof}
Let $k = d^{\beta}$. Note that when $\Phi_i \leq k$, we have 
\begin{equation} \label{eq:lemmac1_upper_grid}
\Phi_i \leq \mb 1_{\Phi_i \geq 0} + \mb 1_{\Phi_i \geq 1} + \cdots + \mb 1_{\Phi_i \geq k}. 
\end{equation}
We again define $\vartheta_j \doteq \sum_{i=1}^d \mb 1_{\Phi_i \geq j}$ and $p_j \doteq \prob{\Phi_1 \geq j}$, then we have
\begin{equation}
\prob{\vartheta_k = 0} = \left(1-p_k\right)^d \geq \exp\left(2dp_k\log(1/2)\right) \geq \exp\left(-Cd^{1-\beta}\right), 
\end{equation}
where the second inequality above follows from the fact $\log\left(1-y\right) \geq 2y \log\left(1/2\right)$ for $y \in \left[0, 1/2\right]$. Moreover we note that 
\begin{equation} \label{eq:lemmac1_expect_ieq}
\expect{\vartheta_j \; \vert \; \vartheta_k = 0} \leq \expect{\vartheta_j} = dp_j, 
\end{equation}
so we have 
\begin{align}
 \; & \prob{\sum_{i=1}^d \Phi_i > t  \; \vert\; \vartheta_k = 0} \\
 \leq \; & \prob{\sum_{i=1}^d \sum_{j=0}^k \mb 1_{\Phi_i \geq k} > t  \; \vert\; \vartheta_k = 0} \quad \text{($\vartheta_k = 0$ implies $\Phi_i \leq k$ for all $i$, and~\eqref{eq:lemmac1_upper_grid})}\\
\leq \; & \prob{d + \sum_{j=1}^k \vartheta_j > t \; \vert \; \vartheta_k = 0} \quad \text{(exchange summation order and substitute into $\vartheta_k$)} \\
\leq \; & \frac{d + \sum_{j=1}^k \expect{\vartheta_j \; \vert \;  \vartheta_k = 0}}{t} \quad \text{(by Markov inequality and linearity of expectation)}\\
\leq \; & \frac{d + d \sum_{j=1}^k p_j}{t} \quad \text{(by~\eqref{eq:lemmac1_expect_ieq})}\\
\leq \; & \frac{d + d/2 + 2d/\pi \integral{1}{k}{x^{-1}}{dx}}{t} \text{(substitute $p_j$ and upper bound finite sum by integral)}\\
= \; & \frac{\frac{2}{\pi}\beta d\log d + \frac{3}{2}d}{t}. 
\end{align}
We set 
\begin{equation}
t = \left(1 + \frac{1}{\log d}\right) \left(\frac{2}{\pi}\beta d\log d + \frac{3}{2}d\right) = \frac{2}{\pi}\beta d\log d + O\left(d\right). 
\end{equation}
Then we have 
\begin{align}
\prob{\sum_{i=1}^d \Phi_i \leq t} 
\geq \; & \prob{\vartheta_k = 0} \prob{\sum_{i=1}^d \Phi_i \leq t \; \vert \; \vartheta_k = 0} \\
\geq \; & \exp\left(-C d^{1-\beta}\right) \left(1-\frac{\log d}{1+\log d}\right), 
\end{align}
yielding the result. \qquad 
\end{proof}
\section{Proof of Lemma~\ref{lemma:lipschitz}} \label{app:proof_lips}

We will need the definition of well-conditioned basis and some existence lemma to proceed. 
\begin{definition}[Well-Conditioned Basis for Subspaces~\cite{dasgupta2008sampling}]
Let $\mc S$ be a $r$-dimensional linear subspace in $\R^D$. For $p \in [1, \infty)$, let $\|\cdot\|_q$ be the dual norm of $\|\cdot\|_p$. Then a matrix $\mb U \in \R^{D\times r}$ is $\left(\alpha, \beta, p\right)$ -well-conditioned basis for $\mc S$ if: (1) columns of $\mb U$ are linearly independent; (2) $\|\mb U\|_p \leq \alpha$; and (3) $\forall \mb z \in \R^r$, $\|\mb z\|_q \leq \beta \|\mb U \mb z\|_p$. $\mb U$ is said to be a $p$-well-conditioned basis for $\mc S$ if $\alpha$ and $\beta$ are $r^{\mc{O}\left(1\right)}$ (i.e., polynomial in $r$) and independent of $D$.   
\end{definition}

The next lemma asserts the existence of $1$-well-conditioned basis for any $r$-dimensional subspaces, justified by the existence of the Auerbach basis.

\begin{lemma}[Existence of $1$-Well-Conditioned Basis, ~\cite{dasgupta2008sampling}]\label{lemma:exist_wcb}
For any linear subspace $\mc S$ of dimension $r$, there exists a $\left(r, 1, 1\right)$-well-conditioned basis. 
\end{lemma} 

\begin{proof}[{\bf of Lemma~\ref{lemma:lipschitz}}]
Fix a $1$-well-conditioned basis $\mb A$ for $\mc S$. Suppose that 
\begin{equation} \label{eqn:desired-ineq} 
\sum_{j=1}^{r+1} \norm{\mb P\mb A_{*j}}{1} \leq t \sum_{j=1}^{r+1} \norm{\mb A_{*j}}{1}.
\end{equation} 
Since any vector $\mb w \in \mc S$ can be written as $\mb w = \mb A \mb x$ for some $\mb x \in \mathbb{R}^{r + 1}$, 
\begin{align}
\norm{\mb P\mb w}{1} 
& = \norm{\mb P\mb A \mb x}{1} = \norm{\mb P\sum_{j=1}^{r+1} \mb A_{*j} x_j}{1} \leq \sum_{j=1}^{r+1} \left|x_j\right| \norm{\mb P \mb A_{*j}}{1} \\
& \leq \norm{\mb x}{\infty} \sum_{j=1}^{r+1} \norm{\mb P \mb A_{*j}}{1} 
\leq \norm{\mb x}{\infty} t\sum_{j=1}^{r+1} \norm{\mb A_{*j}}{1} \\
& \leq \norm{\mb A \mb x}{1} t\left(r+1\right) = t \left(r+1\right) \norm{\mb w}{1}, 
\end{align}
where the last inequality follows from the definition of $1$-well-conditioned basis. Hence, whenever \eqref{eqn:desired-ineq} holds, $L \le t ( r + 1)$, and so
\begin{equation}
\mathbb{P}\left[ L > t (r+1) \right] \le \mathbb{P} \left[ \sum_{j=1}^{r+1} \norm{\mb P\mb A_{*j}}{1} \ge t \sum_{j=1}^{r+1} \norm{\mb A_{*j}}{1} \right].
\end{equation}
We finish by upper bounding the probability on the right hand side, which we define as $\varpi$. For all $i \in [d],  j \in [r+1] $, let $\Psi_{i,j} = \left|\mb P_{i*} \mb A_{*j}\right|/\norm{\mb A_{*j}}{1}$. Obviously $\Psi_{i, j}$'s are all Half-Cauchy RV's and also $\Psi_{i,j}$'s indexed by the same $j$ are independent. Now 
\begin{align}
   \varpi = \prob{\sum_{j=1}^{r+1} \norm{\mb P\mb A_{*j}}{1} \ge t \sum_{j=1}^{r+1} \norm{\mb A_{*j}}{1}} 
= \prob{\sum_{j=1}^{r+1}\left(\norm{\mb A_{*j}}{1} \sum_{i=1}^d  \Psi_{i, j}\right)\geq t \sum_{j=1}^{r+1} \norm{\mb A_{*j}}{1}}. 
\end{align}
Next we partition the probability space and relax a bit to obtain 
\begin{align}
\varpi = \; & \prob{\sum_{j=1}^{r+1}\left(\norm{\mb A_{*j}}{1} \sum_{i=1}^d  \Psi_{i, j}\right) \geq t \sum_{j=1}^{r+1} \norm{\mb A_{*j}}{1} \; \vert \; \exists \; \Psi_{i, j} > B}\prob{\exists \; \Psi_{i, j} > B} +  \nonumber \\
& \quad \prob{\sum_{j=1}^{r+1}\left(\norm{\mb A_{*j}}{1} \sum_{i=1}^d  \Psi_{i, j}\right) \geq t \sum_{j=1}^{r+1} \norm{\mb A_{*j}}{1} \; \vert \; \Psi_{i, j} \leq B, \forall\; i, j} \prob{\Psi_{i, j} \leq B, \forall\; i, j} \\
\leq \; & \prob{\exists \; \Psi_{i, j} > B} +  \nonumber \\
& \quad \prob{\sum_{j=1}^{r+1}\left(\norm{\mb A_{*j}}{1} \sum_{i=1}^d  \Psi_{i, j}\right) \geq t \sum_{j=1}^{r+1} \norm{\mb A_{*j}}{1} \; \vert \; \Psi_{i, j} \leq B, \forall\; i, j} \prob{\Psi_{i, j} \leq B, \forall\; i, j}. 
\end{align}
Applying union bound to the first term and Markov inequality to the conditional probability in the second term, we have 
\begin{align}
\varpi \leq \; & \frac{2d\left(r+1\right)}{\pi B} + \frac{\sum_{j=1}^{r+1}\left(\norm{\mb A_{*j}}{1} \expect{\sum_{i=1}^d  \Psi_{i, j} \; \vert \; \Psi_{i, j} \leq B, \forall\; i, j}\right)}{t \sum_{j=1}^{r+1} \norm{\mb A_{*j}}{1}} \prob{\Psi_{i, j} \leq B, \forall\; i, j} \\
= \; & \frac{2d\left(r+1\right)}{\pi B} +  \frac{\left(\sum_{j=1}^{r+1}\norm{\mb A_{*j}}{1}\right) \expect{\sum_{i=1}^d  \Psi_{i, j} \; \vert \; \Psi_{i, j} \leq B, \forall\; i, j}}{t \sum_{j=1}^{r+1} \norm{\mb A_{*j}}{1}} \prob{\Psi_{i, j} \leq B, \forall\; i, j}\\
= \; & \frac{2d\left(r+1\right)}{\pi B} +  \frac{d\; \expect{\Psi_{1, 1} \; \vert \; \Psi_{1, 1} \leq B}}{t } \prob{\Psi_{i, j} \leq B, \forall\; i, j}, 
\end{align}
where we in the last step we take $j = 1$ with loss of generality as $\Psi_{i, j}$'s are iid half Cauchy for any fixed $j$. We now define a new RV $\Psi_{1, 1}^B$ as: 
\begin{equation}
\Psi_{1, 1}^B = \begin{cases}
\Psi_{1, 1}  & \Psi_{1, 1} \leq B \\
0    & \Psi_{1, 1} > B
\end{cases}, 
\end{equation}
and note the fact that $\expect{\Psi_{1, 1} \; \vert \; \Psi_{1, 1} \leq B} = \expect{\Psi_{1, 1}^B }/\prob{\Psi_{1, 1} \leq B}$, hence 
\begin{align}
\varpi \leq \; \frac{2d\left(r+1\right)}{\pi B} +  \frac{d\; \expect{\Psi_{1, 1}^B }}{t \prob{\Psi_{1, 1} \leq B} } \prob{\Psi_{i, j} \leq B, \forall\; i, j} 
\leq \; \frac{2d\left(r+1\right)}{\pi B} + \frac{d\; \expect{\Psi_{1, 1}^B }}{t}, 
\end{align}
where we have used the fact $\prob{\Psi_{i, j} \leq B, \forall\; i, j} \leq  \prob{\Psi_{1, 1} \leq B}$. 
We arrive at the claimed results by substituting the expectation 
\begin{equation}
\expect{\Psi_{1, 1}^B} = \frac{2}{\pi} \integral{0}{B}{\frac{x}{1+x^2}}{dx} = \frac{1}{\pi} \left. \log\left(1+x^2\right)\right|_{0}^B = \frac{1}{\pi} \log\left(1+B^2\right). 
\end{equation}
This completes the proof. 
\end{proof}
\section{Summing up: Proof of Theorem~\ref{thm:main}} \label{app:proof_main}

\begin{proof}[{\bf of Theorem~\ref{thm:main}}] By Lemma \ref{lemma:good_subspace}, with probability at least $c$,
\begin{eqnarray}
d_{\ell^1}(\mb P \mb q, \mb P \mc S_1)\; \le\;  \left(\frac{2}{\pi} d \log d\right)  d_{\ell^1}(\mb q, \mc S_1).
\end{eqnarray}
We apply Lemmas \ref{lemma:contract} and \ref{lemma:lipschitz} to obtain a probabilistic lower bound on $d_{\ell^1}( \mb P \mb q, \mb P \mc S_j )$ for each $j = 2 \dots n$. As above, let $\tilde{\mc S}_j = \mc S_j \oplus \{ \mb q \}$ denote the direct sum of $\mc S_j$ and the query point. Let $N_j$ denote an $\eps$-net for the intersection of $\tilde{\mc S}_j$ with the $\ell^1$ ball, with size at most $( 3 / \eps )^{r+1}$. Standard arguments guarantee the existence of such a net. 

Applying Lemma \ref{lemma:contract} to each of the $N_j$, we obtain that
\begin{equation}
\norm{\mb P \mb w}{1} \;\ge\; (1-\alpha)(1-\delta) \frac{2}{\pi} d \log d 
\end{equation}
for every $\mb w \in N_j$ and every $j \in \{ 2, \dots, n \}$, simultaneously, with probability at least 
\begin{equation}
1 - \left(n-1\right) \left(\frac{3}{\eps}\right)^{r+1} d^{1-\alpha} \exp\left(-\frac{\delta^2}{2\pi}d^{\alpha}\right).
\end{equation}
At the same time, applying Lemma \ref{lemma:lipschitz} to each $\tilde{S}_j$, we obtain that 
\begin{equation}
\norm{ \mb P \mb w }{1} \;\le\; t ( r+ 1 ) \norm{\mb w }{1}
\end{equation}
simultaneously for every $\mb w \in \tilde{\mc S}_j$, for each $j \in \{ 2, \dots, n \}$, with probability at least 
\begin{equation}
1 - \frac{2d\left(r+1\right)\left(n-1\right)}{\pi B} {-} \frac{2d\left(n-1\right)}{\pi t} \log \sqrt{1+ B^2}.
\end{equation}
Here, $B > 0$ can be chosen freely to obtain the tightest possible bound on the probability of failure. For notational convenience, write 
\begin{equation}
\xi = \frac{ t( r+ 1) \eps }{\frac{2}{\pi} d \log d }, 
\end{equation}
and notice that on the intersection of the good events introduced above,   for every $\mb h \in \tilde{\mc S}_j$ with $\norm{\mb h}{1} \le \eps$, 
\begin{equation}
\norm{\mb P \mb h }{1} \;\le\; \left(\frac{2}{\pi} d \log d \right)\, \xi.
\end{equation}
Consider an arbitrary $\mb w \in \mc S_j$. We can write
\begin{equation}
\frac{\mb q - \mb w}{\norm{\mb q - \mb w}{1}} = \mb z + \mb h, 
\end{equation}
with $\mb z \in N_j$, $ \mb h \in \tilde{\mc S}_j$, and $\norm{\mb h}{1} \le \eps$. Applying $\mb P$ to both sides and using the triangle inequality, we obtain that 
\begin{eqnarray}
\norm{\mb P \mb q - \mb P \mb w}{1} &\ge& \left( \norm{\mb P \mb z }{1} - \norm{\mb P \mb h }{1} \right) \norm{\mb q - \mb w }{1} \nonumber \\
&\ge& \left( \frac{2}{\pi} d \log d \right) \, \left( ( 1- \alpha ) ( 1- \delta ) - \xi \right) \, \norm{\mb q - \mb w }{1}. 
\end{eqnarray}
Hence, on the intersection of the good events introduced above, for each $j = 2 \dots n$, 
\begin{eqnarray}
d_{\ell^1}(\mb P \mb q, \mb P \mc S_j) &\ge& \left( \frac{2}{\pi} d \log d \right) \, \left( ( 1- \alpha ) ( 1- \delta ) - \xi \right) \, d_{\ell^1}(\mb q, \mc S_j ) \nonumber \\
&\ge& \left( \frac{2}{\pi} d \log d \right) \, \left( ( 1- \alpha ) ( 1- \delta ) - \xi \right) \, \eta \, d_{\ell^1}(\mb q, \mc S_1 ) \nonumber \\
&\ge& \left( ( 1- \alpha ) ( 1- \delta ) - \xi \right) \, \eta \, d_{\ell^1}(\mb P \mb q, \mb P \mc S_1 ).
\end{eqnarray}
So, as long as
\begin{equation} \label{eqn:success-condition}
(1-\alpha)(1-\delta) - \xi > 1/\eta,
\end{equation}
the algorithm will succeed, except on an event of probability at most 
\begin{align}
\phi \;\doteq\; 1 - c & + \left(n-1\right) \left(\frac{3}{\eps}\right)^{r+1} d^{1-\alpha} \exp\left(-\frac{\delta^2}{2\pi}d^{\alpha}\right) \nonumber \\
& + \frac{2d\left(r+1\right)\left(n-1\right)}{\pi B} + \frac{2d\left(n-1\right)}{\pi t} \log \sqrt{1+ B^2}. \label{eqn:failure-prob}
\end{align}
Our remaining task is to show that with the specified choice of $d$, \eqref{eqn:success-condition} is satisfied, and the failure probability $\phi$ in \eqref{eqn:failure-prob} is bounded away from one by a constant. 

Set $\zeta = 1 - \frac{1}{\eta} - \alpha$. By assumption, $\zeta > 0$. We will set $\delta = \zeta / 3$, and ensure that $\xi \le \zeta / 3$, which will imply that 
\begin{equation}
(1-\alpha)(1-\delta) - \xi  \;\ge\; 1 - \alpha - 2 \zeta / 3 \;>\; 1/ \eta,
\end{equation}
ensuring that \eqref{eqn:success-condition} is satisfied. We choose 
\begin{equation}
B = \frac{4}{c} \left(\frac{2 d \left(r+1\right)\left(n-1\right)}{\pi}\right), \quad t = \frac{4}{c}\left( \frac{2}{\pi} d\left(n-1\right)\right) \cdot 2 \cdot 4 \cdot \log \left[\max\left(\frac{8}{c \pi}, d, r+1, n-1\right)\right]. \nonumber
\end{equation}
These choices ensure that the quantity $2 d (r+1) (n-1) / \pi B$ in \eqref{eqn:failure-prob} is at most $c/4$. Moreover, using that $B \ge 16 / \pi \ge (1+\sqrt{5})/2$ and the crude bound $\log \sqrt{1+B^2} \le 2 \log B $ for all $B \ge (1 + \sqrt{5})/2$, we can show that the final term in \eqref{eqn:failure-prob} is at most $c / 4$, giving 
\begin{eqnarray}
\phi &\le& 1 - \frac{c}{2} + \left(n-1\right) \left(\frac{3}{\eps}\right)^{r+1} d^{1-\alpha} \exp\left(-\frac{\delta^2}{2\pi}d^{\alpha}\right)  \nonumber \\
 &=& 1 - \frac{c}{2} {+} \exp\left( -\frac{\zeta^2}{ 18 \pi} d^\alpha + (1-\alpha) \log d + (r+1) \log( 3/\eps) + \log( n-1) \right).
\end{eqnarray}
It remains to choose $\eps$ and bound the exponential term above. We set 
\begin{equation}
\eps \;=\; \left(\frac{2}{\pi} d \log d\right) \; \frac{\zeta}{3} \; \frac{1}{t (r + 1)}.
\end{equation}
This ensures that $\xi = \frac{t (r+1)\eps}{(2/\pi) d \log d} \le \frac{\zeta}{3}$, as promised. Plugging in for $t$, we obtain 
\begin{equation}
\eps \;\ge\; \frac{ C_1 \, c \, \zeta \, \log d }{ ( n -1 ) ( r + 1 ) \log \left[ \max \left( d, r + 1, n - 1 \right) \right]}, 
\end{equation}
where $C_1$ is a numerical constant. Using the assumption that $n > r$, we can simplify this bound to 
\begin{equation}
\eps \;\ge\; \frac{C_2 \, c \, \zeta}{n^2 \, \log n },
\end{equation}
with $C_2$ numerical. The exponential term in \eqref{eqn:failure-prob} is then at most 
\begin{equation}
\exp\left( - C_3 \zeta^2  d^\alpha + (1-\alpha) \log d +C_4 r \log\left( \frac{n}{c \zeta } \right) \right)
\end{equation}
To ensure that this term is bounded by $c / 4$, and hence the probability of failure is bounded away from one by a constant, it suffices to ensure that 
\begin{equation}
d \;\ge\; C_5 \left( \frac{\log d + r \log \left( \frac{n
}{c \zeta} \right) + \log\left( \frac{4}{c} \right)}{ \zeta^2 } \right)^{1/\alpha}.
\end{equation}

\end{proof}
\section{Proof of Theorem~\ref{thm:lower_bound}} \label{app:proof_lower_bound}

From a high level, our proof proceeds by exploiting the approximate subspace search to solve sparse recovery problem. Invoking some known lower bounds for sparse recovery problem, we arrive at the bound as stated in Theorem~\ref{thm:lower_bound}. We first record/show some useful results. 
\begin{proposition}[Number of Measurements for Stable Sparse Recovery, Theorem 5.2 \cite{do2010lower}]\label{prop:sparse_random_lower}
For any constant $C \geq 1$, if any distribution $\mu$ over $\R^{m \times t}$ and any algorithm $\msc A$ obey: $\forall \mb x \in \R^t$ and $\mb A \sim \mu$, $\hat{\mb x} = \msc A \left(\mb A \mb x\right)$ and 
\begin{equation} \label{eq:prop_sparse_random_lower}
\norm{\mb x - \hat{\mb x}}{1} \leq C \min_{\norm{\mb x'}{0} \leq k} \norm{\mb x - \mb x'}{1} 
\end{equation}
with probability at least $p > 3/4$, we must have $m \geq C_1\frac{1}{2+2\log\left(2C+3\right)}k \log t/k$ for some numerical constant $C_1$. 
\end{proposition}

The dependency of $m$ on the approximation factor $C$ is directly extracted from the proof to Theorem 5.2 in~\cite{do2010lower}. 

\begin{proposition}[Approximate Subset Query, Theorem 3.1 \cite{price2011efficient}]\label{prop:subset_query_bound}
There is a randomized sparse binary matrix $\mb A$ with $O\left(\frac{c}{\eps}k\right)$ rows and recovery algorithm $\msc A$, such that $\forall \mb x \in \R^t$ and $\mc S \subset [t]$ with $\left|\mc S\right| = k$, $\mb x' = \msc A\left(\mb A\mb x, \mc S\right) \in \R^t$ has $\mathrm{supp}\left(\mb x'\right) \subset \mc S$ and  
\begin{equation}
\norm{\mb x' - \mb x_{\mc S}}{1} \leq \eps \norm{\mb x - \mb x_{\mc S}}{1}
\end{equation}
with probability at least $1-1/k^c$.\footnote{For any vector $\mb x$, $\mb x_{\Omega}$ is a vector of same length of $\mb x$, with coordinates in $\Omega^c$ set to $0$; ${\mb x_{\bar{\Omega}}}$ is a restriction of $\mb x$ to its subvector indexed by $\Omega$. Similarly for matrices. }
\end{proposition}

\begin{proof}[of Theorem~\ref{thm:lower_bound}]. Consider the following distribution $\mu$ (on $\mb A$) and algorithm $\msc A$ for the $k$-sparse recovery problem as defined in Proposition~\ref{prop:sparse_random_lower}. 
\begin{itemize}
\item $\mu$ is a distribution on $\mb A = \left[\begin{smallmatrix} \mb A_{C} \\ \mb A_{B}  \end{smallmatrix}\right]$, where $\mb A_C$ comprises $\ell$ blocks of projection matrices, $\mb A_C^1, \cdots, \mb A_C^{\ell} \in \R^{m \times t}$, from the same distribution $\nu$, stacked vertically, $\mb A_B \in \R^{m' \times t}$ is a randomized sparse binary matrix with $m' = O\left(\frac{c\ell}{\eps}k\right)$ rows from a distribution that verifies Proposition~\ref{prop:subset_query_bound}. The distribution $\nu$ and parameters $m$, $\ell$, $c$, $\eps$, and $C$ are specified below. 

\item For any $\mb x \in \R^t$, $\msc A$ comprises two steps given $\mb A\mb x$: 
\begin{enumerate}
\item Identifying a subset of coordinates of $\mb x$ that probably contains large (in magnitude) elements. Suppose we target at detecting the support of the largest $k$ elements of $\mb x$. This is equivalent to identifying the nearest, out of the ${t \choose k}$ $k$-dimensional canonical subspaces [spanned by any $k$ of the $t$ canonical basis vectors (i.e., $\mb e_1, \cdots, \mb e_t$)], to $\mb x$ in the sense of $\ell^1$ point-to-subspace distance. \\
\\
Let $\nu$ and $m$ be a distribution-projection dimension pair that satisfies the hypothesis of Theorem~\ref{thm:lower_bound} with the parameter tuples $\left(k, {t \choose k}, \eta_{\min}, \gamma\right)$. In particular, this means if the canonical subspaces and $\mb x$ obey the gap condition dictated by $\eta_{\min}$, given $\mb A_C^i \mb x$, $\forall i \in [\ell]$, we can identify the $k$ significant supports as desired with probability at least $\gamma$. This is not true for all $\mb x$ however. Instead, w.l.o.g. assuming the first canonical subspace is the nearest, consider the following ``partitioning''\footnote{The division may not be partitioning in strictly mathematical sense since $\mc I$ or $\mc J$ may be empty. } of canonical subspaces $\mc S_1, \cdots, \mc S_{{t \choose k}}$: 
	\begin{itemize}
	\item $\left\{1\right\}$
	\item $\mc I \doteq \left\{\kappa \in \left[ {t \choose k}\right]\setminus \left\{1\right\}: d_{\ell^1}\left(\mb x, \mc S_{\kappa}\right) < \eta_{\min} d_{\ell^1}\left(\mb x, \mc S_1\right) \right\}$
	\item $\mc J \doteq \left\{\kappa \in \left[ {t \choose k}\right]\setminus \left\{1\right\}: d_{\ell^1}\left(\mb x, \mc S_{\kappa}\right) \geq \eta_{\min} d_{\ell^1}\left(\mb x, \mc S_1\right) \right\}$
	\end{itemize}
Then $\mc I = \emptyset$ corresponds to cases when distance gap $\eta_{\min}$ is obeyed, so $\forall i \in [\ell]$
\begin{equation}
\prob{\mathop{\arg\min}_{\kappa \in \left[ {t \choose k}\right]} d_{\ell^1}\left(\mb A_C^i \mb x, \mb A_C^i \mc S_{\kappa}\right) \in \left\{1\right\} \cup \mc I} = \prob{\mathop{\arg\min}_{\kappa \in \left[ {t \choose k}\right]} d_{\ell^1}\left(\mb A_C^i \mb x, \mb A_C^i \mc S_{\kappa}\right) = 1} \geq \gamma. 
\end{equation}
 If $\mc J = \emptyset$, $\forall i \in [\ell]$
\begin{equation}
1 = \prob{\mathop{\arg\min}_{\kappa \in \left[ {t \choose k}\right]} d_{\ell^1}\left(\mb A_C^i \mb x, \mb A_C^i \mc S_{\kappa}\right) \in \left\{1\right\} \cup \mc I} \geq \gamma. 
\end{equation}
When $\mc I \neq \emptyset$ and $\mc J \neq \emptyset$, we consider in addition a spurious set $\mc I'$ with $\left|\mc I'\right| = \left|\mc I\right|$, which consists of random duplicates of subspaces in $\mc J$. So in this case  
\begin{align}
	& \prob{\mathop{\arg\min}_{\kappa \in \left[ {t \choose k}\right]} d_{\ell^1}\left(\mb A_C^i \mb x, \mb A_C^i \mc S_{\kappa}\right) \in \left\{1\right\} \cup \mc I} \\
\geq \; & \prob{\mathop{\arg\min}_{\kappa \in \left\{1\right\} \cup \mc J} d_{\ell^1}\left(\mb A_C^i \mb x, \mb A_C^i \mc S_{\kappa}\right) = 1} \\
\geq \; & \prob{\mathop{\arg\min}_{\kappa \in \left\{1\right\} \cup \mc J \cup \mc I'} d_{\ell^1}\left(\mb A_C^i \mb x, \mb A_C^i \mc S_{\kappa}\right) = 1} \geq \gamma. 
\end{align}
So in any case $\mb A_C^i$, $i \in [\ell]$ is enough to guarantee a constant probability of success $\gamma$, to identify one subspace that is within $\eta_{\min}$ of the best in terms of distance to $\mb x$. Denote the corresponding supports identified by the $\ell$ independent runs by $\Omega_i$, $\forall i \in [\ell]$ and $\Pi \doteq \cup_{i=1}^{\ell} \Omega_i$, we have 
\begin{equation}
\prob{\exists\; \mc S \subset \Pi: \left|\mc S\right| = k, \norm{\mb x_{\mc S^c}}{1} \leq \eta_{\min} \min_{\left|\mc T\right| = k}\norm{\mb x_{\mc T^c}}{1}} \geq 1 - \left(1-\gamma\right)^{\ell}. 
\end{equation} 
We choose 
\begin{equation}
\ell = -\log 5/\log\left(1-\gamma\right)
\end{equation} 
to make this probability at least $4/5$. 

\item Estimating the value of $\mb x$ on the support from Step $1$. We denote $k' = \left|\Pi\right| \leq k\ell$. Given $\Pi$, by Proposition~\ref{prop:subset_query_bound}, we can obtain an $\hat{\mb x}$ with $\mb A_{B}\mb x$ that obeys: $\mathrm{supp}\left(\hat{\mb x}\right) \subset \Pi$, and 
\begin{equation}
\norm{\hat{\mb x} - \mb x_{\Pi}}{1} \leq \eps \norm{\mb x - \mb x_{\Pi}}{1}
\end{equation}
with probability at least $15/16$, provided 
\begin{equation}
k^c \geq 16 \Longrightarrow c \geq \log 16/ \log k. 
\end{equation}
\end{enumerate} 
\end{itemize}

Putting together above constructions, with probability at least $4/5 \times 15/16 = 3/4$, $\hat{\mb x}$ above satisfies  
\begin{align}
\norm{\hat{\mb x} - \mb x}{1} 
& = \norm{\hat{\mb x} - \mb x_{\Pi}}{1} + \norm{\mb x_{\Pi^c}}{1} \leq \eps \norm{\mb x - \mb x_{\Pi}}{1} + \norm{\mb x_{\Pi^c}}{1} \\
& \leq \left(1+\eps\right)\norm{\mb x_{\Pi^c}}{1} \leq \left(1 + \eps\right) \eta_{\min}\min_{\norm{\mb x'}{0} \leq k} \norm{\mb x - \mb x'}{1}. 
\end{align}
Hence this $\left(\mu, \msc A\right)$ pair respects the hypothesis in Proposition~\ref{prop:sparse_random_lower} and so $\mb A$ must have at least $C_1 k\log\left(t/k\right)/\left[2 + 2 \log \left(2 \left(1+\eps\right) \eta_{\min} + 3\right)\right]$ rows for some constant $C_1$, or each $\mb A_C^i$ must have 
\begin{align}
& \frac{1}{\ell}\left[\frac{C_1 k\log\left(t/k\right)}{2 + 2 \log \left(2 \left(1+\eps\right) \eta_{\min} + 3\right)} - C_2 \frac{c\ell k}{\eps} \right] \nonumber \\
= \; &  C_1' \frac{1}{2 + 2 \log \left(2 \left(1+\eps\right) \eta_{\min} + 3\right)}\log \frac{1}{1-\gamma} k\log\left(t/k\right) - C_2' \frac{k}{\log k} \frac{1}{\eps}. 
\end{align}
rows, for some constants $C_2$, $C_1'$ and $C_2'$. Note that we have $n = {t \choose k}$ subspaces in each subspace search problem, hence by taking $\eps = 1/2$ (corresponding to requiring $C = 1.5\eta_{\min}$ approximation for the $k$-sparse recovery problem we started with) we have $d \geq C_3 \frac{1}{\log 3\left(\eta_{\min} + 1\right)}\log\frac{1}{1-\gamma}\log n - C_4 \frac{k}{\log k}$ for some numerical constants $C_3, C_4$, or translating to the parameter of Theorem~\ref{thm:lower_bound}: 
\begin{equation}
d \geq C_3 \frac{1}{\log 3\left(\eta_{\min} + 1\right)}\log\frac{1}{1-\gamma}\log n - C_4 \frac{r}{\log r}. 
\end{equation}
\\
\\
On the other hand, consider the ${D \choose r}$ canonical subspaces $\left\{\mc S_1, \cdots, \mc S_{{D \choose r}}\right\}$ spanned by any $r$ subset of the canonical basis $\left\{\mb e_1, \cdots, \mb e_D\right\}$. Let $\mb 0 \neq \mb q \in \mc S$, where $\mc S$ is another $r$-dimensional subspace and $\mc S \neq \mc S_i$, $\forall i \in [{D \choose r}]$ and moreover $\mb q \notin \mc S_i, \forall i$. Note that in this case $t = 1$ and $\eta = \infty$. For any projection matrix $\mb P \in \R^{d \times D}$, $\mb P\mb q$ is either $\mb 0$ or spans a $1$-dimensional subspace. 
\begin{itemize}
\item To identify the original subspace unambiguously with nontrivial probability (i.e., better than random guess in any case of ties), $\mb P\mb q$ cannot be zero, as  $\forall i$, $\mb P\mc S_i$ is again a subspace. 
\item When $\mb P \mb q \neq \mb 0$, a necessary condition for unambiguous identifiability is $\mb P \mb q \notin \mb P \mc S_i$, $\forall i$, or 
\begin{equation}
\mb P \mb q \neq \mb P_{\mc S_i}\mb y, \forall \mb y \in \R^r, \forall i \in \left[{D \choose r}\right], 
\end{equation}
where $\mb P_{\mc S_i}$ is the submatrix indexed by the canonical basis vectors associated with the subspace $\mc S_i$. Equivalently, 
\begin{equation} \label{eq:lower_bound_d_main}
\mb P_{\mc S_i^c} \mb q_{\mc S_i^c} \neq \mb P_{\mc S_i}\mb y, \forall \mb y \in \R^r, \forall i \in \left[{D \choose r}\right]. 
\end{equation}
\end{itemize} 
If $m \leq r$, then by rank argument, $\exists i \in [{D \choose r}]$, such that $\mathrm{span}\left(\mb P_{\mc S_i}\right) = \mathrm{span}\left(\mb P\right)$, and hence $\mb P_{\mc S_i^c} \mb q_{\mc S_i^c} \in \mathrm{span}\left(\mb P_{\mc S_i}\right)$, or $\exists \mb y \in \R^r$, such that $\mb P_{\mc S_i^c} \mb q_{\mc S_i^c} = \mb P_{\mc S_i}\mb y$, contradicting~\eqref{eq:lower_bound_d_main}. So we must have $d \geq r$. \qquad 
\end{proof}

\bibliographystyle{siam}
\bibliography{L1}
\end{document}